\newif\ifcomments  
\commentsfalse
\documentclass{article}

\usepackage{amsmath,amssymb,pifont}
\usepackage{multicol}
\usepackage{amstext}
\usepackage{amsthm}
\usepackage{multirow}
\usepackage{booktabs}
\usepackage[skip=0pt]{subcaption}
\usepackage{times}
\usepackage{lipsum}
\usepackage[shortlabels]{enumitem}
\usepackage{cancel}
\usepackage{wrapfig}
\usepackage{array}
\usepackage{siunitx}
\usepackage{csvsimple}
\usepackage[multidot]{grffile}
\usepackage{bbm}
\usepackage{makecell}
\usepackage{bbm, dsfont}
\usepackage{mathtools}
\usepackage[dvipsnames]{xcolor}
\usepackage{comment}

\usepackage{geometry}
\usepackage{mathpazo}
\usepackage{enumitem}

\usepackage[multiple]{footmisc}
\usepackage{mathrsfs}
\usepackage{todonotes}
\usepackage{tikz}

\usepackage{algpseudocode,algorithm,algorithmicx}
\usepackage{xfrac}

\usepackage{graphicx} %
\usepackage{color}

\usepackage{array}
\usepackage{amssymb}
\usepackage{amsmath}
\usepackage{xspace}
\usepackage{fancyhdr}
\usepackage{comment}
\usepackage{listings}
\usepackage[normalem]{ulem}
\usepackage{hyperref}
\usepackage{cleveref}
	
\hypersetup{final}

\newcommand{\eps}{\ensuremath{\varepsilon}}

\newcommand{\bfA}{\ensuremath{\mathbf{A}}}

\newcommand{\bfx}{\ensuremath{\mathbf{x}}}
\newcommand{\bfy}{\ensuremath{\mathbf{y}}}
\newcommand{\bfz}{\ensuremath{\mathbf{z}}}

\newcommand{\calD}{\ensuremath{\mathcal{D}}}

\newcommand{\calG}{\ensuremath{\mathcal{G}}}

\newcommand{\calP}{\ensuremath{\mathcal{P}}}

\newcommand{\calS}{\ensuremath{\mathcal{S}}}

\renewcommand{\Pr}{\mathop{\mathbf{Pr}}}

\newtheorem{lem}{Lemma}[section]
\newtheorem{thm}{Theorem}[section]
\newtheorem{theorem}[thm]{Theorem}

\newtheorem{defn}{Definition}[section]

\newtheorem{claim}{Claim}[section]
\newtheorem{conj}{Conjecture}[section]

\Crefname{lem}{Lemma}{Lemmas}
\Crefname{thm}{Theorem}{Theorems}
\Crefname{claim}{Claim}{Claims}
\Crefname{defn}{Definition}{Definitions}

\makeatletter
\newcommand{\vast}{\bBigg@{4}}
\newcommand{\Vast}{\bBigg@{5}}
\makeatother

\newcommand{\ex}[2]{{\ifx&#1& \mathbb{E} \else
\underset{#1}{\mathbb{E}} \fi \left[#2\right]}}
\newcommand{\pr}[2]{{\ifx&#1& \mathbb{P} \else
\underset{#1}{\mathbb{P}} \fi \left[#2\right]}}

\usepackage{ulem}

\DeclarePairedDelimiterX{\infdivx}[2]{(}{)}{%
  #1\;\delimsize\|\;#2%
}

\newcommand{\mypar}[1]{\smallskip
	\noindent{\textbf{{#1}:}}}
	
\renewcommand{\epsilon}{\varepsilon}

\renewcommand{\tilde}{\widetilde}

\setlist{nolistsep}
\setlist[itemize]{noitemsep, topsep=0pt}

\setlist{nolistsep}
\setlist[itemize]{noitemsep, topsep=0pt}

\def\MC{\textup{MC}}
\def\IS{\textup{IS}}
\newcommand{\score}{\texttt{score}}
\def\csp{\textup{CSP}}
\newcommand{\maxcut}[1]{\textup{MAX-}{#1}\textup{-CUT}}
\newcommand{\threelin}[1]{3\textup{LIN}({#1})}

\def\ZZ{\mathbb{Z}}

\def\NN{\mathbb{N}}
\def\maxind{\textup{MAX-IND-SET}}
\def\naux{n_{\textup{aux}}}
\def\maxcuttwo{\textup{MAX-CUT}}

\def\approxratio{0.987}
\def\soundness{0.9238}
\def\completeness{0.9349}
\def\speedupfour{10,000\times}
\def\speedupthree{10,000\times}
\def\parallelcopies{1429}
\def\tspforced{\textup{MWST}_f}

\lstdefinestyle{mypython}{
  language=Python,
  backgroundcolor=\color{gray!10},
  basicstyle=\ttfamily\footnotesize,
  keywordstyle=\color{blue},
  commentstyle=\color{green!50!black},
  stringstyle=\color{red!60!black},
  showstringspaces=false,
  frame=single,
  breaklines=true
}

\usepackage{listings}
\usepackage{longtable}  
\newcommand{\codesn}{\mathcal{C}}

\newcommand{\wthreelintwo}{\text{Hybrid-}\threelin{2}}
\newcommand{\TRUE}{\texttt{True}\,}
\newcommand{\FALSE}{\texttt{False}\,}
\newcommand{\clause}{\zeta}
\newcommand{\predicate}{\Gamma}
\newcommand{\alphabet}{\ZZ_k}
\newcommand{\complete}{c}
\newcommand{\sound}{s}
\newcommand{\localcomps}{u_{co}^{\text{local}}}
\newcommand{\mcst}{\textup{MWST}}
\newcommand{\predeqv}[3]{\texttt{P}^{\equiv #1}_{#3, #2}}
\newcommand{\predneq}{\texttt{P}^{\neq}_2}
\newcommand{\inst}{\texttt{I}}
\newcommand{\instj}{\texttt{J}}
\newcommand{\insti}[3]{\inst^{\equiv #1}_{#3, #2}}
\newcommand{\countcl}{\texttt{count}}
\newcommand{\mainalphabet}{\Omega} 
\newcommand{\predeq}{\texttt{P}^{=}_2}
\newcommand{\insteq}{\texttt{I}^{=}_2}
\newcommand{\instneq}{\texttt{I}^{\neq}_2}
\newcommand{\edgeeq}{\texttt{E}^{=}_2}
\newcommand{\edgeneq}{\texttt{E}^{\neq}_2}
\newcommand{\edgei}[3]{\texttt{E}^{\equiv #1}_{#3, #2}}

\newcommand{\edgeteq}{\texttt{T}^{=}_2}
\newcommand{\edgetneq}{\texttt{T}^{\neq}_2}
\newcommand{\edgeti}[3]{\texttt{T}^{\equiv #1}_{#3, #2}}
\newcommand{\edgeueq}{\texttt{U}^{=}_2}

\usepackage{fullpage}
\usepackage[numbers, sort, comma, square]{natbib}

\title{Reinforced Generation of Combinatorial Structures: Hardness of Approximation}
\author{
Ansh Nagda\thanks{University of California, Berkeley, and Google DeepMind.}
\and
Prabhakar Raghavan\thanks{Google.}
\and
Abhradeep Thakurta\thanks{Google DeepMind.}
}
\date{}
 
\begin{document}

\maketitle
\thispagestyle{empty}
\begin{abstract}
    Can AI based methods help us make advances in complexity theory? We provide evidence towards answering this in the affirmative, using AlphaEvolve (an LLM code mutation agent) to obtain new results in three settings:

a) Average-case hardness for MAX-CUT and MAX-Independent Set: We improve a recent result of Kunisky and Yu to obtain near-optimal upper and (conditional) lower bounds on certification algorithms for MAX-CUT and MAX-Independent Set on random 3- and 4-regular graphs. Our improved lower bounds are obtained by constructing nearly extremal Ramanujan graphs on as many as $163$ vertices. Additionally, via analytical arguments we strengthen the upper bounds to settle the computational hardness of these questions up to an error in the third decimal place.


{
    b) Worst-case Hardness of Approximation for MAX-k-CUT: We obtain new inapproximability results, proving that it is NP-hard to approximate $\maxcut{4}$ and $\maxcut{3}$ within factors of $\approxratio$ and $0.9649$ respectively, using AlphaEvolve to discover new gadget reductions. Our $\maxcut{4}$ result improves upon the SOTA of $0.9883$, and our $\maxcut{3}$ result improves on the current best gadget-based inapproximability result of $0.9853$, but falls short of improving the SOTA of $16/17$ that relies on a custom PCP (rather than a gadget reduction from “standard” Håstad-style PCPs). 
}  

c) Worst-case Hardness of Approximation for the metric Traveling Salesman Problem (TSP): We show that it is NP-hard to approximate the minimum cost tour within a factor of $111/110$ using AlphaEvolve to discover a new gadget, thus improving the SOTA of $117/116$. We provide a modular soundness and completeness argument for the reduction of $\threelin{2}$ (a standard constraint satisfaction problem used in hardness reductions) to TSP, which enabled AlphaEvolve to search over finite constraint graphs. This modularization may be of independent interest for future work on TSP inapproximability.

A key technical challenge we faced: verifying a candidate construction produced by AlphaEvolve is costly (sometimes requiring time exponential in the size of the construction). Our results were enabled by using AlphaEvolve itself to evolve the verification procedure to be faster (sometimes by $10,000\times$ for our gadgets). Our results suggest that gadget based proofs would benefit from a pass through AI-based tools to obtain stronger results.

\end{abstract}
\clearpage
\newpage
\thispagestyle{empty}
\tableofcontents
\clearpage
\newpage
\pagenumbering{arabic}
\newpage

\section{Introduction}
\label{sec:intro}
In this paper, we study the following question: Can AI based methods help us make novel and non-trivial discoveries in complexity theory? In particular, we seek to focus on classical and well-studied problems to derive results that stand the test of time. We are interested in proofs and constructions that will plausibly defy non-AI methods (whether by direct human effort, or through traditional computational methods such as SAT solvers).
We provide evidence towards answering this in the affirmative, via progress on three problems.
\begin{itemize}
    \item[a)]~\emph{Hardness of certifying upper bounds on the $\maxcuttwo$ and $\maxind$ of sparse random graphs.} In~\Cref{sec:avhardness}, we improve upon the combinatorial structures presented in \cite{kunisky2024computational}, achieving improved lower bounds in the $\{3,4\}$-regular cases. Complementing this result, we obtain new upper bounds which improve upon Hoffman's classic spectral bounds~\cite{hoffman2003eigenvalues,haemers2021hoffman}. (These upper bounds are derived via analytical approaches\footnote{These improvements are derived from a discussion with Sidhanth Mohanty, Northwestern University.}.) Together, these achieve an almost tight resolution to these problems.\footnote{By this we mean: the lower bounds and upper bounds on efficient certification algorithms match up to the second decimal place.}. 
    \item[b)]~\emph{NP-Hardness of approximation for $\maxcut{k}$.} In~\Cref{sec:maxkcut}, we give new gadget-based reductions (from standard H\aa stad PCPs~\cite{haastad2001some}) for the NP-hardness of approximating $\maxcut{k}$ for $k\in \{3,4\}$. For $\maxcut{4}$, we improve on the current best inapproximability result from $0.9883$~\cite{austrin2014new} to $\approxratio$. For $\maxcut{3}$, we improve on the current best gadget-based inapproximability result from $0.9853$~\cite{kann1996hardness} to $0.9649$. Our result for $\maxcut{3}$ is sandwiched between two results using custom PCPs for this problem:~\cite{guruswami2009improved} with $0.9696$, and~\cite{austrin2014new} with $0.9411$. 
    \item[c)]~\emph{NP-Hardness of approximation for  metric TSP.} In~\Cref{sec:tsp}, we give a new gadget based reduction (from standard H\aa stad PCPs~\cite{haastad2001some}) for the NP-hardness of approximating metric TSP. We improve the current SOTA of $117/116$~\cite{chlebik2022weighted} to $111/110$. Our improvement advances the long line of work establishing hardness results for this  problem~\cite[$1+\delta$]{papadimitriou1993traveling} (without an explicit value for $\delta$),~\cite[$5381/5380$]{engebretsen1999explicit},~\cite[$3183/3182$]{bockenhauer2000improved},~\cite[$220/219$]{papadimitriou2006approximability},~\cite[$185/184$]{lampis2014improved},~\cite[$123/122$]{karpinski2015new},~\cite[$117/116$]{chlebik2022weighted}. 
\end{itemize}

Notably, our results rely on a single (AI-derived) technique, AlphaEvolve~\cite{novikov2025alphaevolve, romera2024mathematical}, to find and verify finite combinatorial structures that improve on prior constructions. 
At a high level AlphaEvolve uses an LLM (Large-Language-Model) to iteratively evolve code snippets that generate combinatorial structures (we will sometimes refer to these as \textit{constructions}) that are of high quality by some criterion. Typically this process uses a synthetic objective function that is a \textit{proxy} for the true objective in the theorem being sought; for instance for $\maxcut{k}$ and the TSP we seek to improve the constant in the hardness theorems. AlphaEvolve cannot directly optimize this constant, because of the manner in which proofs embed gadgets; rather, AlphaEvolve improves gadgets in a ``good direction'' that we specify and then --- given a generated gadget --- we reconstruct the resulting hardness guarantee. These mechanics will become evident in Sections~\ref{sec:maxkcut} and \ref{sec:tsp}.
While the structures we generate using AlphaEvolve are finite constructions, our main results (Theorems \ref{thm:avg-case-lb},~\ref{thm:max-3-cut}, and~\ref{thm:tsp}) embed the universal quantifier $\forall n$, with $n$ being the size of the problem instance, thanks to appropriate ``lifting''~\cite{kunisky2024computational,trevisan2000gadgets,chlebik2022weighted} arguments.  At a minimum our work suggests that we would do well to routinely run gadget proofs through an ``AI optimization'' phase. We present our results in increasing order of human involvement (required to make the problem amenable to AlphaEvolve).

Operationally, the system consists of three parts: a) a code snippet ($\codesn$) that constructs combinatorial structures, b) An evaluation function (referred to as the~\emph{verifier}) that verifies and scores the generated structures, and c) An LLM that suggests a new code snippet $\codesn_{\texttt{new}}$ based on the set of prior $\codesn$'s and the history of constructions. Via prompting the LLM, the objective is to nudge $\codesn_{\texttt{new}}$ towards building better structures. (Appendix~\ref{sec:backgroundAE} provides a more detailed overview of AlphaEvolve.)

The efficacy of AlphaEvolve crucially depends on the verification step which in turn conditions the search landscape. Because we seek extremal structures in a large search space,~\emph{fast} verification (including scoring) of a construction helps AlphaEvolve  try a large number of combinatorial structures, and learn patterns from them to prune the search space. The eventual sizes of the structures we find for our problems are directly correlated with the speed of verification.

\mypar{Speeding up verification in AlphaEvolve} The combinatorial structures sought in these problems are based on undirected graphs. While brute-force verification of our structures was reasonably fast for the TSP lower bound and for certifying upper bounds on sparse random graphs, it was significantly more challenging for $\maxcut{k}$. We discuss this in more detail below.

For TSP, one of the discovered gadgets had $20$ edges (over $12$ vertices) along with a large number of constraints ($\approx 11!$) for a successful verification. As a result, we needed some care in writing a speedy verifier that operated within roughly one second.

For the hardness of certifying upper bounds for sparse random graphs, we use AlphaEvolve to search for $d=\{3,4\}$-regular Ramanujan graphs, along with a witness for a large cut/independent set. Evaluating whether a graph is Ramanujan, and whether or not the cut/independent set is a valid witness, are efficient. Consequently, for these problems we could find and verify graphs with as many as $163$ vertices. 

{In studying the NP-Hardness of approximation for $\maxcut{k}$, we seek gadgets that reduce the $\threelin{k}$ problem~\cite{haastad2001some} to \maxcut{k} ($\threelin{k}$ corresponds to a set of constraints of the form $x+y+z \equiv b \pmod k$, where $x,y,z\in [0,\ldots,k-1]$). Here the verification is unfortunately not efficient. It requires checking the gadget against $\Omega(k^m)$ linear constraints, where $m$ is the total number of variables in the target problem. A standard technique in gadget-based hardness reductions is to add auxiliary variables to have flexibility in mapping the source and the target problems~\cite{trevisan2000gadgets}. This exacerbates the exponential running time of the verifier (for reference, our final gadgets contain $m=14$ and $m=19$ variables for $k=3$ and $k=4$ respectively). To address this we invoke AlphaEvolve to speed up the verifier, as follows. 

We prompted AlphaEvolve to improve the execution time of the verifier, while ensuring that it passes various correctness checks. To our surprise, AlphaEvolve improved the running time by $\speedupfour$ for our final $\maxcut{4}$ gadget of size $m=19$, and by a similar factor for our final $\maxcut{3}$ gadget of size $m=14$. We used human inspection to verify the correctness of these improved verifiers (given their possibility of being correct only on synthetically generated gadgets used for correctness checks). The improvement came from a number of system-level improvements together with a branch-and-bound strategy to prune many of the $\Omega(k^m)$ constraints. We discuss this in more detail in~\Cref{sec:fasterVerification}.


We emphasize that our theorems do not depend on the correctness of these fast verifiers; even though our AlphaEvolve computations internally use them, the gadgets found by AlphaEvolve were verified by a brute-force algorithm that explicitly checks all $\Omega(k^m)$ constraints. 
}
In Section~\ref{sec:discussion} we argue that these constructions are unlikely to have been found by hand, through conventional computational techniques such as SAT solvers, or by ``directly prompting'' a language model (the method most widely reported in the literature~\cite{} as evidence that AI can help research in mathematics). Instead our approach here is to invoke AlphaEvolve to autonomously find and verify the structures used in these proofs. In part we are unable to assess the robustness of direct prompting techniques because -- to date -- in many cases it is unclear how much human effort went into ``prompt engineering'' the model, relative to the contribution of the model. We are unable to reproduce our AlphaEvolve-derived results through a fair amount of such prompt engineering.


\section{Hardness of Certification Problems on Random Graphs}
\label{sec:avhardness}

A central problem in average case complexity is that of \emph{certification}, where our goal is to certify a property of typical samples from a distribution. We focus on the setting of \emph{sparse random graphs}, where the distribution in question is $\calG(n, d)$, the uniform distribution over all $n$-vertex $d$-regular graphs, where $d\in \mathbb{N}$ is small. Let $\Omega$ be the set of $n$-vertex graphs, and consider a \emph{property} $P:\Omega\to \{0,1\}$ such that $\Pr_{G\sim \calG(n, d)}[P(G) = 1] = 1-o_n(1)$. We are given as input a graph $G\sim \calG(n, d)$, and our algorithmic task is to efficiently certify that $G$ satisfies $P$ with high probability.

\begin{defn}
    An efficient certificate of $P$ is any property $P':\Omega\to \{0,1\}$ that is computable in polynomial time such that the following conditions hold. 
    \begin{enumerate}
        \item $P'(G) = 1$ implies $P(G) = 1$.
        \item $\Pr_{G\sim \calG(n,d)}[P'(G)=1] = 1-o_n(1)$.
    \end{enumerate}
    Here, $\calG(n,d)$ is the uniform distribution over all $n$-vertex $d$-regular graphs for some fixed constant $d\in \mathbb{N}$.
    \label{def:certcoNP}
\end{defn}

We focus specifically on the setting of \emph{refutation}, where $P$ is a co-NP property (recall that a problem is in co-NP if the \texttt{NO} instance of a decision problem has an efficiently verifiable certificate). This setting has seen considerable activity, leading to innovation in various algorithmic techniques such as spectral and sum-of-squares algorithms, along with the development of a variety of methods to establish hardness (for instance, \cite{barak2019nearly,raghavendra2017strongly,bandeira2021spectral,kunisky2024computational,kothari2023planted}).

In this work, we consider the problem of certifying upper bounds on the $\maxcuttwo$ (or $\maxind$) of a graph drawn from $\calG(n,d)$. For example, for $\maxcuttwo$, we study the following question: for what values of $\sigma$ can we efficiently certify with high probability that a random graph drawn from $\calG(n,d)$ has $\maxcuttwo$ at most $\sigma \times \#(\texttt{of edges})$?

\begin{defn}[Max-Cut fraction, Max-Independent Set fraction]
    For an undirected graph $G$, $\MC(G)$ is defined as the fraction of edges in a maximum cut of $G$.
    Similarly, $\IS(G)$ is defined as the fraction of vertices in a maximum independent set of $G$.
\end{defn}

\begin{defn}\label{def:certification-problem-def}
    $\sigma^{\MC}_d$ is the infimum over all $\sigma$ such that the property $\mathbf{1}\{\MC(G)\leq \sigma\}$ has an efficient certificate.
    Similarly, $\sigma^{\IS}_d$ is the infimum over all $\sigma$ such that $\mathbf{1}\{\IS(G)\leq \sigma\}$ has an efficient certificate.
\end{defn}

There is evidence that this is a nontrivial problem exhibiting a \emph{certification gap}, in the sense that $\MC(G)$ concentrates around a constant value strictly smaller than $\sigma^{\MC}_d$ for $G\sim \calG(n,d)$ \cite{bandeira2021spectral,kunisky2024computational}. These represent the structurally most elementary refutation-style properties (involving only binary variables and pairwise constraints) that exhibit interesting complexity-theoretic phenomena. In this work, we focus on accurately characterizing the value of $\sigma^{\MC}_d$ and $\sigma^{\IS}_d$. We remark that the asymptotics of these quantities as $d\to \infty$ are already understood (thanks to \cite{bandeira2021spectral,kunisky2024computational}). In service of a precise understanding of this problem in all regimes, we focus on the least understood case where $d$ is small; specifically we study $d\in \{3,4\}$.

Towards this goal,~\cite{kunisky2024computational} introduced a conjecture\footnote{We refer the reader to \cite{kunisky2024computational} for supporting evidence and intuition.} that relates $\sigma^{\MC}_d$ and $\sigma^{\IS}_d$ to a fundamental question in spectral graph theory, specifically about Ramanujan graphs (\Cref{def:ram} below), which are (in a spectral sense) the most ``random'' looking regular graphs constructed via a deterministic process. 

The intuition behind this framework relies on the strategy of "quiet planting". Specifically, if one can exhibit a Ramanujan graph failing to satisfy a certain property $P$ (e.g., a small $\maxcuttwo$ value), then under the hardness assumption of \cite[Conjecture 1.6]{kunisky2024computational}, it is impossible for a polynomial-time algorithm to certify $P$ on the random graph distribution $\mathcal{G}(n,d)$. In the context of $\maxcuttwo$, if we construct a Ramanujan graph with a maximum cut of at least $\gamma^{\MC}_d \times \#(\texttt{edges})$, then the certification bound $\sigma^{\MC}_d$ must essentially be at least $\gamma^{\MC}_d$, where $\gamma^{\MC}_d$ is defined in~\Cref{def:mcr} below. Therefore, the objective in establishing stronger lower bounds for certification is to maximize this value $\gamma^{\MC}_d$ over the set of Ramanujan graphs. We employ AlphaEvolve to construct such extremal Ramanujan graphs for both $\maxcuttwo$ and $\maxind$. The precise conjecture we operate with is stated in~\Cref{conj:avcase}.

\begin{defn}[Ramanujan graphs]
    Let $G$ be a $d$-regular multigraph on $n$ vertices with adjacency matrix $A$. Suppose the eigenvalues of $A$ are $\lambda_1\geq \lambda_2\geq \ldots \geq \lambda_n$. We say that $G$ is Ramanujan if $\max_{i>1}|\lambda_i|\leq 2\sqrt{d-1}$.
    \label{def:ram}
\end{defn}

\begin{defn}
    Define $\gamma^{\MC}_d$ (resp. $\gamma^{\IS}_d$) as the supremum of $\MC(G)$ (resp. $\IS(G)$) over all $d$-regular Ramanujan graphs $G$.
    \label{def:mcr}
\end{defn}

$\gamma^{\MC}_d$ (resp. $\gamma^{\IS}_d$) can be thought of as the performance of the optimal ``spectral'' certificate for bounds on $\maxcuttwo$ (resp. $\maxind$).

\begin{conj}[\cite{kunisky2024computational}]
    $\sigma^{\MC}_d \geq \gamma^{\MC}_d$ and $\sigma^{\IS}_d \geq \gamma^{\IS}_d$ for all $d\geq 3$.
    \label{conj:avcase}
\end{conj}

~\cite{kunisky2024computational} exhibited finite $d$-regular Ramanujan graphs for $d\in \{3,4\}$, yielding concrete lower bounds on $\sigma_d^{\MC}$ and $\sigma_d^{\IS}$ that are strong enough to prove a certification gap. On the other hand, there was a large gap between the best known upper bound~\cite{hoffman2003eigenvalues,haemers2021hoffman} and their lower bounds (see \Cref{tab:comparison} for a comparison). Our results make progress towards a more precise understanding of this question. Firstly, we use AlphaEvolve to construct better $d$-regular Ramanujan graphs, exhibiting stronger lower bounds.

\begin{theorem}[Lower bound]\label{thm:avg-case-lb}
    We have the lower bounds $\gamma_4^{\MC}\geq 113/124$, $\gamma_3^{\IS}\geq 17/36$, and $\gamma_4^{\IS}\geq 74/163$. Under \Cref{conj:avcase}, this implies the following computational hardness results: $\sigma_4^{\MC}\geq 113/124>0.911$, $\sigma_3^{\IS}\geq 17/36>0.472$, and $\sigma_4^{\IS}\geq 74/163>0.453$.
\end{theorem}

This Theorem improves upon the results of \cite{kunisky2024computational}, who prove $\gamma_4^{\MC}\geq 7/8$, $\gamma_3^{\IS}\geq 11/24$, and $\gamma_4^{\IS}\geq 3/7$. An explicit description of the constructions found by AlphaEvolve can be found in Appendix~\ref{sec:avg-appendix-lb}, and a figure of the construction achieving the $\gamma_4^{\MC}$ bound is presented in \Cref{fig:ramanujan}. We complement these hardness results with an algorithmic improvement over Hoffman's classic bound~\cite{hoffman2003eigenvalues,haemers2021hoffman}.

\begin{theorem}[Upper bound]\label{thm:avg-case-ub}
    We have  $\sigma_3^{\MC}\leq 0.953$, $\sigma_4^{\MC}\leq 0.916$, $\sigma_3^{\IS}\leq 0.476$, and $\sigma_4^{\IS}\leq 0.457$.
\end{theorem}

\begin{proof}[Proof ideas for the $\maxcuttwo$ upper bound]
    Similar to existing bounds, we will use the spectral bound $\max_{i>1}|\lambda_i|\lesssim 2\sqrt{d-1}$, that is satisfied with high probability when $G\sim \calG(n,d)$. Here $\lambda_1\geq \lambda_2\geq\ldots\geq \lambda_n$ are the eigenvalues of the adjacency matrix $\bfA$ of $G$.

    Consider a cut $(S, \bar{S})$ in $G$ such that $\alpha$ fraction of edges cross $(S,\bar{S})$. Hoffman's bound proceeds by constructing the vector $\bfx = (2\cdot \mathbf{1}\{i\in S\} - 1, i\in [n])$. The spectral bound above implies an upper bound on $|\bfx^\top \bfA \bfx|$, which in turn constrains $\alpha$. We generalize this approach by considering inequalities involving new higher-order quadratic forms of the form $\bfx^\top \bfA^L \bfx$, where $L\in \NN$. It is much less clear how bounds on $\bfx^\top \bfA^L \bfx$ imply better bounds on $\alpha$; we show that the strongest bound one can derive on $\alpha$ can be captured by a linear program over probability distributions on $\{\pm 1\}$-labelings of a $d$-regular tree of depth $L$. Such a distribution is meant to model the local view of $(S,\bar{S})$ experienced by a random vertex. We solve this LP for small fixed values of $L>1$, finding that one can improve on Hoffman's bound (corresponding to $L=1$).
\end{proof}

We prove this result in Appendix~\ref{sec:avg-appendix-ub}. For reference, we have collected our results and how they compare to previous results in \Cref{tab:comparison}. Notably, we obtain upper and lower bounds on $\sigma_4^{\MC}$, $\sigma_3^{\IS}$, and $\sigma_4^{\IS}$ that match up to a small absolute error of $0.005$, so both \Cref{thm:avg-case-lb,thm:avg-case-ub} are nearly optimal for these cases. In particular, this raises the exciting possibility that one might be able to obtain clues toward complete characterizations of $\sigma_d^{\MC}$ and $\sigma_d^{\IS}$ by studying the proofs of \Cref{thm:avg-case-lb,thm:avg-case-ub}.

\begin{table}[t]
\centering
\begin{tabular}{lcc|cc}
\toprule
 & LB~(\cite{kunisky2024computational})&  LB (\Cref{thm:avg-case-lb}) & UB (\Cref{thm:avg-case-ub}) & UB (\cite{hoffman2003eigenvalues,haemers2021hoffman}) \\
\midrule
$\maxcuttwo, d=3$ & 0.944 & 0.944 & {\bf\color{blue}  0.953} & 0.971 \\
$\maxcuttwo, d=4$ & 0.875 & {\bf\color{blue} 0.911} & {\bf\color{blue}  0.916} & 0.933 \\
$\maxind, d=3$ & 0.458 & {\bf\color{blue}  0.472} & {\bf\color{blue}  0.476} & 0.485 \\
$\maxind, d=4$ & 0.428 & {\bf\color{blue}  0.454} & {\bf\color{blue}  0.457} & 0.464 \\
\bottomrule
\end{tabular}
\caption{Comparison of hardness (LB) and algorithms (UB) for certifying upper bounds on $\maxcuttwo$ and $\maxind$ for $\calG(n,d)$, in terms of lower bounds and upper bounds on $\sigma_d^{\MC}$ or $\sigma_d^{\IS}$ (see \Cref{def:certification-problem-def}). The LB results are conditional on \Cref{conj:avcase}. Our improvements are {\bf\color{blue}highlighted}.}\label{tab:comparison}
\end{table}

\begin{figure}[t]
  \centering
    \includegraphics[scale=0.2,trim=0cm 2cm 0cm 2.5cm, clip]{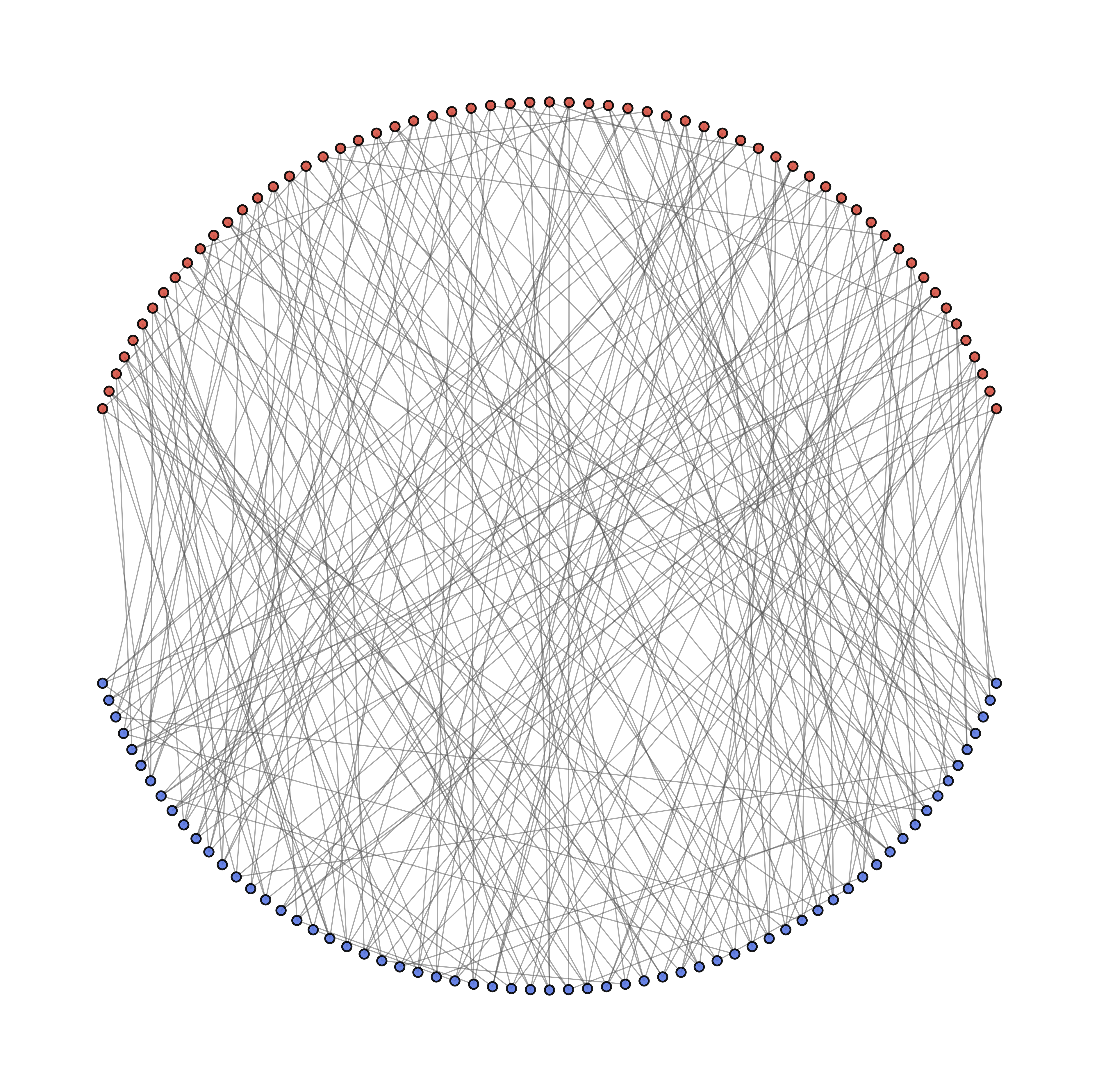}
    \caption{$4$-regular Ramanujan graph found by AlphaEvolve for the lower bound on $\gamma_4^{\MC}$ in \Cref{thm:avg-case-lb}.}
    \label{fig:ramanujan}
\end{figure}

\paragraph{Comments on finding near-optimal lower bounds with AlphaEvolve.} 

We conclude this section with remarks on our methodology in \Cref{thm:avg-case-lb}, and our results. The constructions given in \cite{kunisky2024computational} for $d\in \{3,4\}$ are graphs on up to $12$ vertices that were generated by computer-assisted experimentation. We were able to replicate their results by randomly sampling a large number of regular graphs, albeit with the post-hoc knowledge of the number of self-loops in their construction.

Improving their lower bounds necessitates constructions on many more vertices, as the granularity of $\MC(G)$ or $\IS(G)$ is limited by the size of $G$. At our target scale, the approach of randomly sampling constructions does not work for two reasons: (1) the space of $d$-regular $n$-vertex graphs blows up combinatorially, meaning random sampling will not find interesting ``extremal'' graphs, and (2) the complexity of computing $\MC(G)$ or $\IS(G)$ is exponential in $n$, so it is hard to even compute these bounds for a construction.

As stated before, we used AlphaEvolve as a tool to search for finite Ramanujan graphs witnessing the bounds in \Cref{thm:avg-case-lb}. We circumvent the above limitation by searching over the space $\calS = \bigcup_{n=1}^\infty \Omega_n\times 2^{[n]}$, where $\Omega_n$ is the set of $n$-vertex $d$-regular graphs. That is, we search directly for a pair $(G, S)$, where $G$ is a graph and $S$ is a cut (or independent set) in $G$, scoring the pair as $\score(G, S) = -\infty$ if $G$ is not Ramanujan, and $\score(G, S) = {|\delta_G(S,\bar{S})|}/{|E(G)|}$ otherwise. Here, $\delta_{G}(S,\bar{S})$ is the set of edges crossing the cut $(S,\bar{S})$ in $G$. This gives an efficiently computable score function that lower bounds $\gamma_d^{\MC}$ (or $\gamma_d^{\IS}$). Consequently, AlphaEvolve isn't limited to the space of small graphs. Indeed, it outputs constructions on as many as $163$ vertices (for $\gamma_4^{\IS}$).

For $d=3$ for $\maxcuttwo$, AlphaEvolve recovered the existing bound $\gamma_3^{\MC}\geq 17/18$ of \cite{kunisky2024computational}, but wasn't able to improve on it. We note that for the sake of conciseness we restricted ourselves to the $d\in \{3,4\}$ cases. We expect our techniques to achieve similar improvements over known bounds on $\gamma_d^{\MC}$ and $\gamma_d^{\IS}$ (\cite{bandeira2021spectral,hoffman2003eigenvalues}) for all reasonably small values of $d>4$ except when $d-1$ is an odd perfect square, when \cite{bandeira2021spectral} already gives the optimal lower bound.
\section{Gadget based NP-Hardness for Approximating \maxcut{k}}
\label{sec:maxkcut}

The field of \emph{approximation algorithms}~\cite{williamson2011design} concerns finding approximately optimal solutions to computationally hard combinatorial optimization problems. In \emph{hardness of approximation}~\cite{arora2009computational} the main goal is to understand when such approximations are computationally hard. We work within the well-studied framework of Constraint Satisfaction Problems (CSPs).

Consider a collection of predicates $\calP=\{\predicate_1,\ldots, \predicate_r\}$ over some finite alphabet $\mainalphabet$, where the predicates are functions $\predicate_i:\mainalphabet^{l_i}\to \{0,1\}$. A CSP is defined by a set of variables and a collection of clauses, where each clause applies a predicate from $\calP$ to a subset of the variables, and the goal is to find an assignment of values from $\mainalphabet$ to the variables that maximizes the number of satisfied constraints. The hardness of approximating a CSP is parameterized by two parameters $c$ (completeness) and $s$ (soundness).  In this paper we will prove statements of the following form: Assuming $\texttt{P}\neq\texttt{NP}$, it is hard to distinguish between CSP instances where the maximum fraction of satisfied clauses is at least $c$, from those where it is at most $s$. This in turn implies that it is $\texttt{NP}$-hard to approximate the maximum number of satisfied clauses in a CSP instance within a factor of $s/c$. In this paper we will parameterize the hardness of CSPs with the tuple $(c,s)$.

\begin{defn}[CSP]
    Let $\mainalphabet$ be a finite alphabet, and let $\predicate_1,\ldots, \predicate_r$ be a collection of predicates over $\mainalphabet$, that is, $\predicate_i:\mainalphabet^{\ell_i}\to \{0,1\}$ for some $\ell_i\in \mathbb{N}$. We define $\csp(\predicate_1,\ldots, \predicate_r)$ to be the collection of instances $\inst:\mainalphabet^n\to \mathbb{R}_{\geq 0}$ of the form $\inst(\bfx) = \sum_{j\in [m]} \clause_j(\bfx)$ where for each $j$, $\clause_j$ is a clause, i.e., it takes the form $\clause_j(\bfx) = \predicate_i(x_{a_1}, x_{a_2},\ldots, x_{a_{\ell_i}})$ for some $i\in [r]$ and distinct $a_1,\ldots, a_{\ell_i}\in [n]$.
    \label{def:CSP}
\end{defn}

In the Max-CSP setting, the input is a description of an instance $I$ of a particular CSP, and the goal is to compute the \emph{maximum} number of constraints that can be satisfied. In other words, we are trying to approximately compute the function $f(\inst) = \max_{\bfx\in \mainalphabet^n}\inst(\bfx)$. It will be useful to cast the approximation problem as a decision problem, as we do below by introducing the notion of $(c,s)$-approximation.

\begin{defn}[$(c,s)$-approximation, $\alpha$-approximation]
    Given a constraint satisfaction problem  $\csp(\predicate_1,\ldots, \predicate_r)$ and constants $0 \leq s\leq c\leq 1$, the $(c,s)$ approximation problem for $\csp(\predicate_1,\ldots, \predicate_r)$ is the problem of, given an $n$-variable instance $\inst=\sum_{j\in [m]}\clause_j$, deciding whether (1) $\max_\bfx \inst(\bfx)\geq c\cdot m$ or (2) $\max_\bfx \inst(\bfx)\leq s\cdot m$. Similarly for $0<\alpha\leq 1$, the $\alpha$-approximation problem is the problem of finding some $\bfx'$ such that $\inst(\bfx')\geq \alpha\cdot \max_x \inst(\bfx)$.
\end{defn}

In particular, an $\alpha$-approximation algorithm for the Max-CSP problem, with $\alpha \geq s/c$,  gives an $(c, s)$-approximation algorithm. The main driving question can be stated as ``given a CSP, what are all pairs $0\leq s \leq c\leq 1$ such that there is an efficient $(c,s)$ approximation for the CSP?'' 

{
We have a strong quantitative understanding for a number of CSPs~\cite{haastad2001some,chan2016approximation} under NP-hardness. However, there are some quite simple CSPs for which very little is known quantitatively. For example, consider $\maxcut{k}$, the problem of partitioning the vertices of a graph into $k$ sets in a way that maximizes the number of edges crossing the partition. This problem can be written as a CSP over alphabet $\mainalphabet = \ZZ_k$ as $\maxcut{k} := \csp(\predneq)$ where $\predneq$ is the ``inequality'' predicate $\predneq(x,y) = \mathbf{1}\{x\neq y\}$. For small values of $k$, there is a large gap between the best known algorithmic results (concretely, $0.836$ and $0.857$-approximations for $k=3$ and $k=4$ respectively~\cite{goemans2001approximation,de2004approximate}) and the best known NP-hardness ($16/17+\eps\approx 0.941$ and $85/86+\eps\approx 0.9883$ for $k=3$ and $k=4$ respectively~\cite{austrin2014new})\footnote{Even under the Unique Games Conjecture, the optimal approximability curve is not known for $k=4$, and for $k=3$ is only known in the low completeness regime~\cite{heilman2023three} (which implies a weak inapproximability of $0.989$).}.
}

\emph{Gadget reductions} provide a methodical framework to find reductions between problems. Specifically in the context of inapproximability, let $\csp_{\text{source}}$ be a problem with known inapproximability, and suppose we would like to prove inapproximability for another problem $\csp_{\text{target}}$. The idea is to systematize the set of reductions from $\csp_{\text{source}}$ to $\csp_{\text{target}}$ by considering only ``local'' reductions that replace each clause of an instance of $\csp_{\text{source}}$ with a collection of new variables, along with an instance of $\csp_{\text{target}}$, known as a \emph{gadget} on the variables involved. Crucially, the gadget for a particular predicate of $\csp_{\text{source}}$ is a finite instance of $\csp_{\text{target}}$ of a fixed constant size.

As an illustrative example, consider the reduction from 3-SAT to MAX-2-SAT appearing in~\cite{trevisan2000gadgets}. Here, the source constraint is a 3-CNF clause $C = x_1 \lor x_2 \lor x_3$. This single constraint is mapped to a gadget consisting of a collection of ten 2-CNF clauses. The local assignment is defined over the~\emph{original variables} $\{x_1, x_2, x_3\}$ along with an~\emph{auxiliary variable} $z$ introduced specifically for this gadget.

The target collection contains the singleton clauses $x_1, x_2, x_3, z$, the pairwise negated clauses of the variables, and mixed clauses linking the variables to the auxiliary. The full set of target clauses is:
\[
\begin{gathered}
x_1, x_2, x_3, \neg x_1 \lor \neg x_2, \neg x_2 \lor \neg x_3, \neg x_3 \lor \neg x_1, \\
z, \neg x_1 \lor \neg z, \neg x_2 \lor \neg z, \neg x_3 \lor \neg z.
\end{gathered}
\]
In this construction, if the source constraint is satisfied, there exists a local assignment (a value for $z$) such that a fraction $c=0.7$ (7 out of 10) of the target clauses are satisfied. Conversely, if the source constraint is not satisfied, no local assignment can satisfy more than a fraction $s=0.6$ of the target constraints.

{
We use AlphaEvolve to find such gadget reductions from the well-understood CSP $\threelin{k}$ to $\maxcut{k}$ for $k\in \{3,4\}$. Here $\threelin{k}$ is defined as $\csp\left(\predeqv{0}{k}{3}, \predeqv{1}{k}{3},\ldots, \predeqv{(k-1)}{k}{3}\right)$ for the $3$-ary ``linear equation'' predicates $\predeqv{i}{k}{3}(x,y,z) = \mathbf{1}\{x + y + z \equiv i\pmod k\}$. {We adopt the notation $\texttt{P}^{\texttt{type}}_{r, k}$, where the superscript specifies the nature of the constraint, $r$ denotes the arity, and $k$ (if present) indicates the alphabet size.}
}


\begin{theorem}\label{thm:max-3-cut}
    For any $\eps>0$, there is a gadget reduction from $(1-O(\eps), 1/3+O(\eps))$-approximating $\threelin{3}$ to $(0.9193-\eps, 0.887+\eps)$-approximating $\maxcut{3}$. As a consequence, it is NP-hard to $0.9649+\eps$-approximate $\maxcut{3}$. More precisely, the completeness and soundness correspond to $(57/62-\eps, 55/62+\eps)$, and the corresponding to a hardness of $55/57+\eps$.
\end{theorem}

\begin{theorem}\label{thm:max-4-cut}
{
    For any $\eps>0$, there is a gadget reduction from $(1-O(\eps), 1/4+O(\eps))$-approximating $\threelin{4}$ to $(\completeness-\eps, \soundness+\eps)$-approximating $\maxcut{4}$. As a consequence, it is NP-hard to $\approxratio+\eps$-approximate $\maxcut{4}$.
}
\end{theorem}

{
We prove the results in Appendix~\ref{sec:gadget-appendix}. In particular, we used AlphaEvolve to search for the gadgets used in these results. Since $\maxcut{k}$ is a symmetric binary predicate, we can view a gadget, which is an instance of $\maxcut{k}$, as an undirected graph (possibly containing parallel edges). In~\Cref{fig:combined_gadgets,fig:fourcut}, we provide a visual representation of these instances. We note that our gadgets look extremely different in the $k=3$ and $k=4$ cases;  in particular, in the $k=3$ case our gadgets contain at most $3$ parallel copies of each edge, while in the $k=4$ case the gadgets contain as many as $\parallelcopies$ parallel copies of some edges. As a result, \Cref{thm:max-4-cut} does not have an inapproximability factor that can be written as a rational number with a small denominator.
}

\begin{figure}[htb]
  \centering
  \begin{subfigure}{0.45\textwidth}
    \centering
    \includegraphics[scale=0.05]{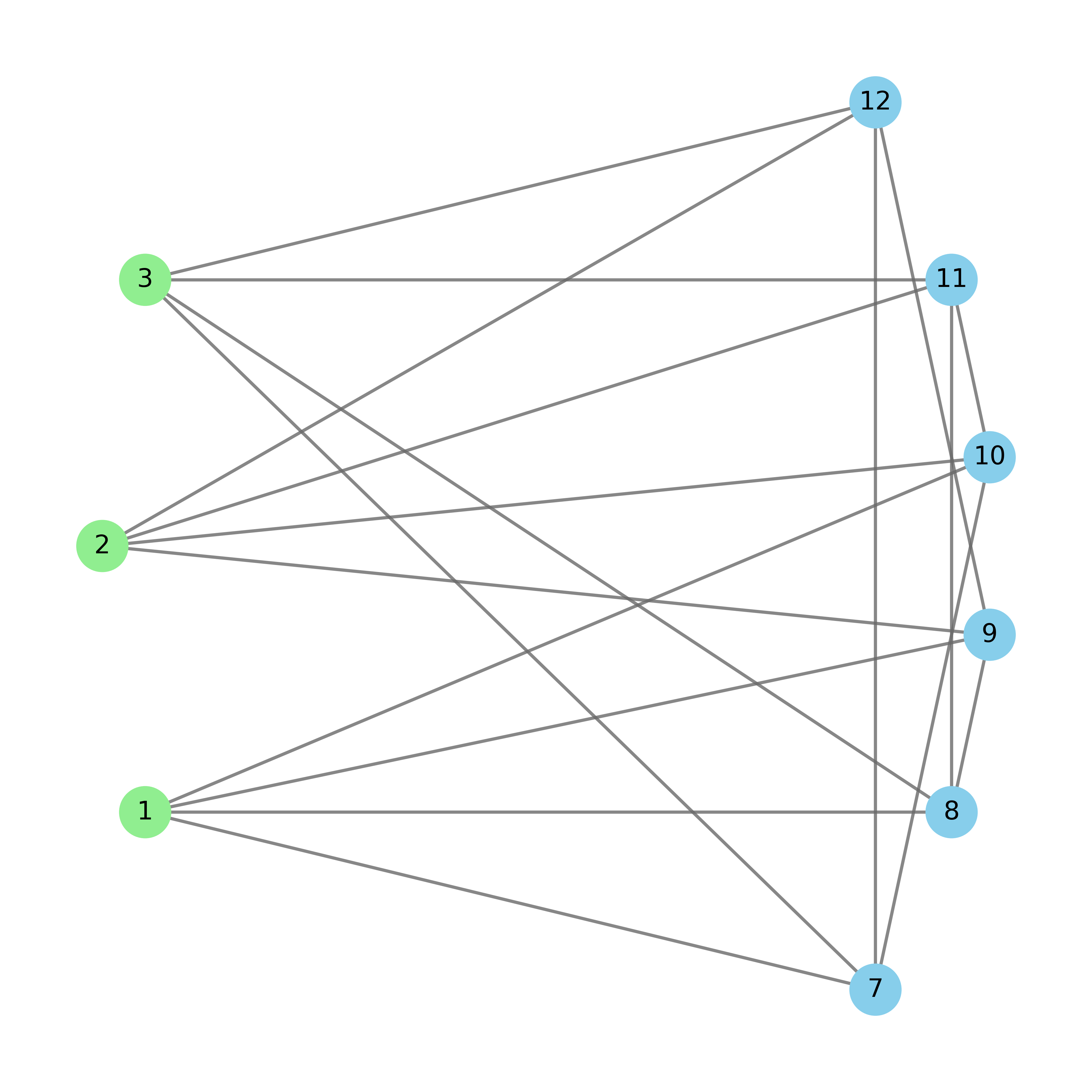}
    \caption{Gadget reducing $\predeqv{0}{3}{3}$ to $\predneq$. The vertices $\{1,2,3\}$ are primary variables, and the rest are auxiliary variables. All edges represent single copies of $\predneq$ clauses. The global variables are not pictured as they are unused.}\label{subfig:gadget-zero}
  \end{subfigure}
  \hfill
  \begin{subfigure}{0.45\textwidth}
    \centering
    \includegraphics[scale=0.05]{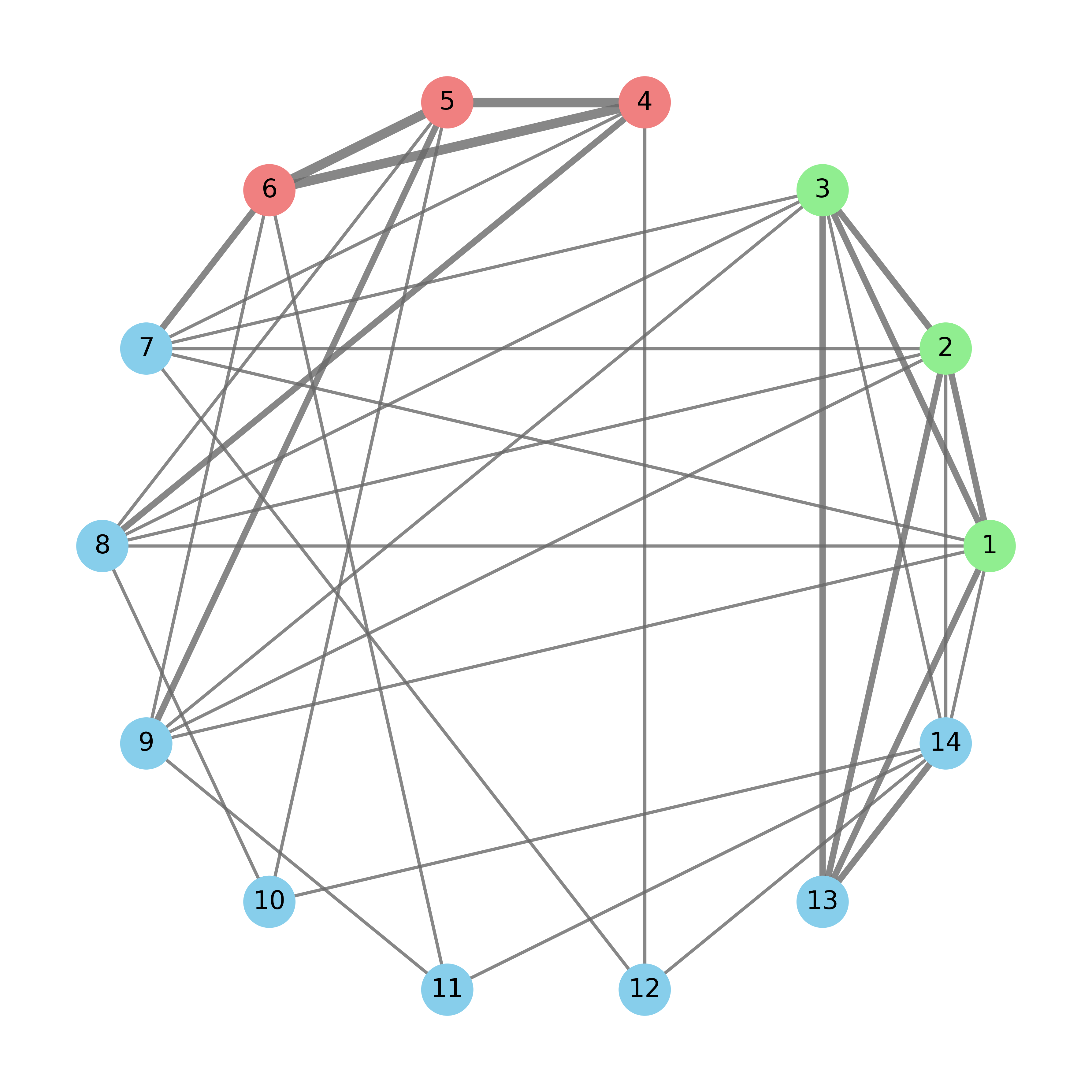}
    \caption{Gadget reducing $\predeqv{1}{3}{3}$ to $\predneq$. The vertices $\{1,2,3\}$ are primary variables, $\{3,4,5\}$ are global variables, and the rest are auxiliary variables. The edges represent one, two, and three copies of the corresponding $\predneq$ clause, depending on the thickness.}\label{subfig:gadget-1}
    \label{fig:second}
  \end{subfigure}
  \caption{Gadgets found by AlphaEvolve for reducing $\threelin{3}$ to $\maxcut{3}$ (see Appendix~\ref{sec:gadget-appendix} for a more explicit description via edge lists). Note that an edge $(i,j)$ in the graph corresponds to the predicate $\predneq$ applied to variables $i$ and $j$, with thickness  proportional to its number of copies in the instance.}
  \label{fig:combined_gadgets}
\end{figure}

\begin{figure}[ht]
  \centering
    \includegraphics[scale=0.25, trim=0cm 1.5cm 0cm 2cm, clip]{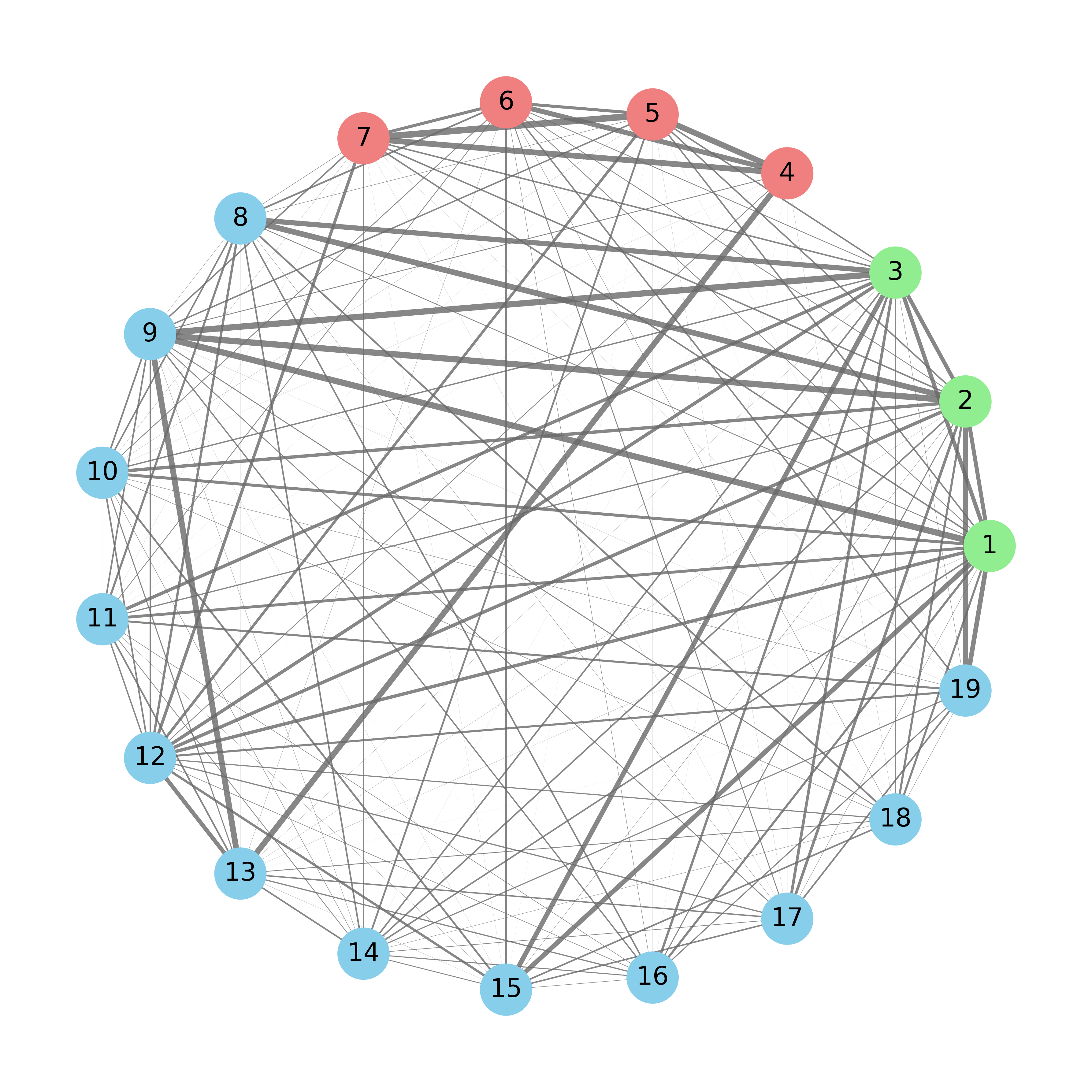}
  \caption{Gadget found by AlphaEvolve for reducing $\threelin{4}$ to $\maxcut{4}$. The thickness of an edge $(i,j)$ is  proportional to the number of copies of the predicate $\predneq$ applied to variables $i$ and $j$ in the gadget, which is between $1$ and $\parallelcopies$ (see Appendix~\ref{sec:gadget-appendix} for a more explicit description via edge lists).}
  \label{fig:fourcut}
\end{figure}

\paragraph{Comparison to prior inapproximability results.}

{
To the best of our knowledge, the current best hardness factor for $\maxcut{4}$ is $85/86+\eps$~\cite[Theorem 1.2]{austrin2014new}, which we improve with \Cref{thm:max-4-cut} to $\approxratio$. The hardness in~\cite{austrin2014new} was obtained by a reduction from $\maxcut{3}$ to $\maxcut{k}$, for all $k>3$.
}

{As for $\maxcut{3}$,} we compare our result in \Cref{thm:max-3-cut} to three works~\cite{kann1996hardness,guruswami2009improved,austrin2014new}, who obtained successively stronger NP-hardness results for $\maxcut{3}$. Our inapproximability ratio of $55/57+\eps$ in \Cref{thm:max-3-cut} beats the $67/68+\eps$ of~\cite{kann1996hardness}\footnote{This result can be deduced by plugging in the ``state of the art'' NP-hardness of approximation for \maxcut{2} \cite{trevisan2000gadgets} into their argument.} and $32/33+\eps$ of~\cite{guruswami2009improved}.

Our result is not strong enough to beat the state of the art of $16/17+\eps$ by Austrin, O'Donnell, Tan, and Wright \cite{austrin2014new}, who use a custom reduction from Label Cover. In contrast, our result does not require any new PCP machinery, utilizing only Håstad's classic PCP \cite{haastad2001some}. We are unaware of a fundamental barrier that limits the gadget based approach from beating the inapproximability ratio of~\cite{austrin2014new}, and it is conceivable that a gadget based reduction from a different source problem to $\threelin{3}$ can achieve this.

\subsection{Developing a Search Framework for Soundness and Completeness Guarantees}
\label{sec:search}

In order to apply AlphaEvolve to this problem, we developed a template for a gadget-based reduction argument and isolated the properties required from the gadget in this argument (see \Cref{defn:gadget} and~\Cref{thm:reduction-general} for details). We scored a candidate gadget by its final performance, that is, the inapproximability ratio that is implied by applying \Cref{thm:reduction-general} to the candidate gadget.

{
This template requires $k$ separate gadgets $\{\insti{i}{k}{3} : 0 \leq i < k\}$, that correspond to the predicates $\{\predeqv{i}{k}{3}:0\leq i < k\}$ of $\threelin{k}$. An important feature of $\maxcut{k}$ is its ${\mathbb{S}}_k$-symmetry -- it is invariant under permuting the alphabet $\ZZ_k$. Unfortunately this feature is not shared by $\threelin{k}$, and as a result, our reductions require a way to break this symmetry of $\maxcut{k}$. We achieve this by adding $k$ new \emph{global variables} to the standard systematization of gadgets~\cite{trevisan2000gadgets}.

\paragraph{Finding gadgets for $\maxcut{3}$.} AlphaEvolve found a good $\insti{0}{3}{3}$ gadget on $9$ variables (\Cref{subfig:gadget-zero}) quite quickly, mainly because the symmetries of $\predeqv{0}{k}{3}$ work well with $\maxcut{k}$ when $k=3$. In particular, this gadget did not require the aforementioned global variables. Building on ideas from the TSSW framework~\cite{trevisan2000gadgets} we can show that this is the optimal $\insti{0}{3}{3}$ gadget in terms of the final inapproximability ratio.

In contrast, finding good $\insti{1}{3}{3}$ and $\insti{2}{3}{3}$ gadgets required significantly more effort and some new ideas. Unlike $\insti{0}{3}{3}$, we found that increasing the number of auxiliary variables consistently led to better performing gadgets, with the only bottleneck being the runtime of computing the performance of a candidate gadget (we comment more on this point in \Cref{sec:fasterVerification}).

The $\insti{1}{3}{3}$ gadget we use in the proof of \Cref{thm:max-3-cut} (\Cref{subfig:gadget-1}) requires $14$ total variables. We have some limited evidence that our $\insti{1}{3}{3}$ gadget is ``locally'' optimal in the sense that it gives the optimal inapproximability result among all gadgets on $14$ variables with the same witness as $\insti{1}{3}{3}$ for all but one satisfying assignment $\bfx$ of $\predeqv{1}{k}{3}$ (see the completeness case of \Cref{defn:gadget} for what is meant by a ``witness'')\footnote{Such a gadget giving the optimal inapproximability result can be quantified as a mixed integer program, which we solved computationally to determine that our gadget is the optimal one. See \Cref{sec:othercompapproach} for more details.}.

All gadgets found for $\maxcut{3}$ had between 0 and 3 copies of each edge. This is expected, as optimal gadgets found in many previous results (e.g., \cite{trevisan2000gadgets,haastad2017improved}) contain small integer weights. As a result, we get a clean inapproximability ratio of $55/57$.

\paragraph{Finding gadgets for $\maxcut{4}$.} The main difference in our process of finding gadgets for $\maxcut{4}$ was that we had to search over gadgets with more variables (as many as $19$), so we required an extremely well-performing implementation of the verifier (again, we elaborate in \Cref{sec:fasterVerification}). Even with this optimized verifier, evaluations were quite slow, requiring on the order of one second to evaluate a single gadget on $19$ variables. As a result, it took a lot longer for AlphaEvolve to find a search algorithm that, given this optimized verifier, could find \emph{any} nontrivial gadget. The final search algorithm AlphaEvolve produced had the distinctive property of searching over weighted gadgets with real-valued weights, resulting in gadgets with a wide variety of weights between $1$ and $\parallelcopies$ after appropriate scaling and rounding.

}

\subsection{Faster Verification via AlphaEvolve to Explore Larger Gadgets}
\label{sec:fasterVerification}
{
The main challenge with finding large gadgets is that the cost of scoring a gadget scales exponentially in the number of variables; computing the completeness and soundness parameters described in \Cref{defn:gadget} essentially amount to solving an instance of $\maxcut{k}$, which requires exponential time. 
Even for $k=3$, AlphaEvolve slows down significantly when searching for gadgets of size as few as $11$ with a brute force $\maxcut{3}$ verifier.

This problem does not have an off-the-shelf solution in the form of existing fast verifiers; it is unlikely that existing SMT/MIP solvers~\cite{een2003extensible,martins2014open} can be repurposed to solve $\maxcut{k}$. To solve this issue, we used AlphaEvolve itself to speed up a naive brute force implementation of $\maxcut{k}$, scoring a candidate implementation by runtime and correctness. In order to calculate the runtime, we created a synthetic dataset of a wide variety of $\maxcut{k}$ instances, drawn from $20$ random models with varying amounts of planted structure. We then tasked AlphaEvolve to maximize the number of variables $m$ for which the verifier requires at most one second on average to solve instances of size $m$ from our dataset.

The biggest challenge was ensuring that AlphaEvolve does not cheat and find an \emph{incorrect} verifier that is much faster. As mentioned before, we achieved correctness by (1) checking that the verifier is correct on our synthetic dataset, and (2) using a separate judge LLM to certify that a candidate verifier is correct. Each of these techniques was individually too lenient to avoid incorrect verifiers, but we found by human inspection that they were enough in combination.

We note that once $m$ is large enough, it is not possible to label our dataset with the ``ground truth'' scores using a brute-force implementation (which is guaranteed to be correct). Instead, we inductively rely on the correctness of previous verifiers produced by AlphaEvolve (which have already passed the above correctness checks) that are fast enough to provide labels for large $m$. 

We did this separately for the $k=3$ and $k=4$ cases, obtaining verifiers that are optimized to each particular problem. We found that systems-level improvements were most important to the $k=3$ case, with the final verifier offloading the main $O(3^m)$ time computation to a highly optimized tensor contraction operation in~\texttt{numpy}, and only performing slow python-based computation for $O(3^{m/3})$ time. This verifier resulted in a $\speedupthree$ speedup for instances of size $m=14$.

For $k=4$, the final verifier used a more sequential branch-and-bound strategy along with some systems-level improvements. At $m=19$, this provided a $\speedupfour$ speedup against even a \texttt{numba}-accelerated~\cite{lam2015numba} brute force verifier.
}

\subsection{Comparison to other computational techniques}
\label{sec:othercompapproach}

We now survey some other computational techniques to find gadgets and discuss why they appear to be infeasible at the scale of our specific problems, even for the simpler setting of $\maxcut{3}$. 

The most straightforward comparison is the TSSW framework \cite{trevisan2000gadgets}, which casts the task of finding the optimal gadget as a linear program (LP). The main difficulty in doing this is the presence of existential quantifiers in computing the completeness of the gadget. In order to eliminate these quantifiers, \cite{trevisan2000gadgets} canonicalize the auxiliary variables in the gadget. As a result, the size of the LP encoding the optimal gadget is doubly exponential in the number of satisfying assignments of the source predicate; this is $3^{3^6}$ for a reduction from $\threelin{3}$ to $\maxcut{3}$, which is computationally infeasible. Sometimes (as is the case for our $\insti{0}{3}{3}$ gadget for $\maxcut{3}$) it is possible to argue that not all auxiliary variables are required, leading to a more tractable LP, but it is not clear that this is possible outside of very special cases.

If one wants to fix the number of variables in the gadget to a particular constant smaller than $3^{3^6}$ (like $14$ for our gadgets), it is possible to write a mixed integer program (MIP) instead of an LP by encoding the existential constraints using integer variables. For example, solving the MIP took $10$ hours even with all but one existential constraint eliminated\footnote{This experiment was done with Krishnamurthy (Dj) Dvijotham, Google DeepMind.}. {(We used the SCIP solver~\cite{BolusaniEtal2024OO,BolusaniEtal2024ZR}, with a direct encoding of the problem.)} Consequently the MIP approach also seems infeasible with SOTA solvers.
\section{NP-Hardness of Approximating Metric TSP}
\label{sec:tsp}

The Traveling Salesman Problem (TSP)~\cite{williamson2011design} is one of the most studied problems in combinatorial optimization. While there are many variants of this problem~\cite{Karlin2021,Sebo2014,Svensson2020,Traub2019}, we specifically focus on the metric variant of the problem~\cite{Karlin2021}, where the objective is to find a minimum weight Hamiltonian cycle in a weighted, complete, undirected graph with the weights satisfying
the triangle inequality. This in turn corresponds to outputting a permutation over the vertices in the graph that has the minimum weight cycle. In this paper we are concerned with the hardness of approximating metric TSP.  The best approximation algorithm achieves a factor of $1.5 - 10^{-34}$~\cite{gurvits2023trees}. A series of results~\cite{papadimitriou1993traveling, engebretsen1999explicit,bockenhauer2000improved,papadimitriou2006approximability,lampis2014improved,karpinski2015new,chlebik2022weighted} led up to the SOTA NP-hardness of approximation to a factor of $117/116$. 
We use AlphaEvolve to improve this result to $111/110$. 

\begin{thm}
    For any $\varepsilon>0$, it is NP-hard to approximate the metric TSP to  within $111/110-\varepsilon$.
    \label{thm:tsp}
\end{thm}

\mypar{A sparse CSP instance} We follow the framework of \cite{lampis2014improved,karpinski2015new,chlebik2022weighted} and start with an NP-hard CSP we call $\wthreelintwo$, which we reduce to metric TSP to obtain \Cref{thm:tsp}. $\wthreelintwo$ is a particularly structured weighted CSP where every variable appears in exactly~\emph{three} equations, where the equations are either $\threelin{2}$ constraints of the form $x\oplus y\oplus z=b$, or binary linear constraints of the form $x\oplus y = b$, for some $b\in \{0,1\}$. One can show that given a $\wthreelintwo$ instance, it is NP-hard to decide whether (1) most equations can be satisfied, or (2) one cannot satisfy a significant fraction of the equations. (This is the same source problem in~\cite{chlebik2022weighted}, and a formal statement about the hardness of approximation is provided in~\Cref{thm:hybrid_CSP}.) 


A \emph{spanning tour} in a graph is defined as a closed walk that visits each vertex at least once. It is convenient to work with an equivalent formulation of metric TSP~\cite{karpinski2015new} (in terms of inapproximability), which we denote by $\mcst$. The goal there is to find the minimum weight spanning tour in a weighted, undirected graph $G=(V, E, w)$. \cite{chlebik2022weighted} first reduces from an instance of $\threelin{2}$ to an instance  of $\wthreelintwo$ via an expander graph based CSP-to-CSP reduction. Then they reduce from $\wthreelintwo$ to $\mcst$.  Intuitively, the equivalence stems from the triangle inequality, which guarantees that any closed walk visiting vertices multiple times can be ``short-cut'' to form a valid Hamiltonian cycle of equal or lesser cost.

This $\mcst$ formulation is particularly advantageous for hardness reductions. By shifting the focus from finding a simple cycle to finding a minimum-cost connected Eulerian subgraph, one avoids the cumbersome task of explicitly constructing the metric closure. This allows the reduction to define weights locally on the sparse graph—typically assigning small weights to edges and large penalties to non-edges—while ensuring that the structural properties of the original hard instance are preserved.

We use AlphaEvolve to find a gadget (called the equation gadget) that improves the reduction from $\wthreelintwo$ to $\mcst$, while keeping the rest of the arguments intact. This immediately improves the hardness-of-approximation factor from $117/116$ to $111/110$.

\mypar{Reducing the sparse CSP to $\mcst$}  In the gadget reduction from $\wthreelintwo$ to $\mcst$ each variable is replaced by a vertex, and each equation is replaced by a corresponding gadget. An \emph{equation gadget} (more concretely described below) gives a way to encode an assignment of~\TRUE/~\FALSE to the variables in a $\threelin{2}$ equation (i.e., clauses of the form $\predeqv{0}{2}{3}$ or $\predeqv{1}{2}{3}$) into a connectivity property for the corresponding vertices, in the sense that if the equation is satisfied, one can connect the involved vertices to the rest of the graph with low cost; on the other hand, if the equation is unsatisfied, then any spanning tour must dedicate a large total weight of edges in order to connect the vertices involved to the rest of the graph. 

Using AlphaEvolve, we find a new equation gadget (\Cref{fig:tsp_AE_pc1}) achieving better performance than the one appearing in~\cite{chlebik2022weighted} (\Cref{subfig:tsp1_pc0})\footnote{The $\threelin{2}$ equation encoded in the equation gadgets in~\cite{chlebik2022weighted,karpinski2015new,lampis2014improved} all have the form $x\oplus y\oplus z=0$, as opposed to $x\oplus y\oplus z=1$ in Figure~\ref{fig:tsp_AE_pc1}.}. In order to quantify this improvement, we now describe equation gadgets more concretely: each $\threelin{2}$ equation $\predeqv{1}{2}{3}$ (of the form $x\oplus y\oplus z=1$) is assigned an equation gadget, where the green vertices $\{1,2,3\}$ in \Cref{fig:tsp_AE_pc1}, also called the~\emph{contact vertices}, correspond to the variables $\{x,y,z\}$. The red vertex $\{4\}$ is called the~\emph{central vertex}, and is shared across all appearances of the equation gadget across all $\threelin{2}$ equations in the instance. The rest of the vertices in the equation gadget (shown in blue in~\Cref{fig:combined_gadgets}) are used to ensure that there is a gap between the weight of a spanning tour due to a satisfied $\threelin{2}$ equation versus an unsatisfied one. The black edges (referred to as ``unforced edges'') are optional in any tour, while any tour is forced to take each of the red edges (referred to as ``forced edges'') at least once. We note that a standard trick (e.g.,~\cite{karpinski2015new}) can be used to implement the constraint of forcing certain edges to be taken within $\mcst$. The green dashed edges are called~\emph{special edges}, and are purely used for analysis purposes, and do not appear in the actual $\mcst$ instance.

As described earlier, an equation gadget performs well if its contribution to a spanning tour is small /large for satisfied/unsatisfied clauses respectively. In what follows, we will distill this requirement into a self-contained statement about the gadget. Given an equation gadget, we consider disjoint collections of tours within it that cover every vertex. Such a collection is associated with a particular assignment of $(x,y,z)$, where a variable is assigned $\TRUE$ if and only if its corresponding contact vertex appears in the same tour as the central vertex $4$. Furthermore, each tour beyond the one containing the central vertex suffers a weight penalty of one. For an assignment $(x,y,z)$ for a $\threelin{2}$ equation, we will be concerned with the minimum possible total weight of any such collection associated with $(x,y,z)$ (including any weight penalties). We note that the above description is simplified as it does not account for ``dishonest'' collections of tours that interact in unwanted ways with the rest of the reduction. We adopt a slightly different formalism in our formal proofs in order to handle these, deferring the discussion to Appendix~\ref{sec:reductionTSP}. 

\mypar{Our improvement to the equation gadget} 
The improved equation gadget found by AlphaEvolve (\Cref{fig:tsp_AE_pc1}) admits collections having total weight $10$ associated with satisfying assignments for a $\threelin{2}$ equation, and at least $11$ associated with unsatisfying assignments, as opposed to $13$ and $14$ respectively for the gadget in~\cite{chlebik2022weighted}. This immediately improves the performance of the reduction, yielding an improvement in the inapproximability ratio of $\mcst$ (and equivalently metric TSP) from $117/116$ to $111/110$. A description of our new gadget, the full reduction, and a proof of the resulting inapproximability ratio as a function of the equation gadget is provided in Appendix~\ref{app:TSP}.

While we present the equation gadget in~\Cref{fig:tsp_AE_pc1} corresponding to the predicate $x\oplus y\oplus z=1$, we can get the same approximation ratio using a gadget corresponding to $x\oplus y\oplus z=0$ (which is in the same setup of~\cite{chlebik2022weighted}). The gadget is however more complex. We defer its description, and a side-by-side comparison with~\cite{chlebik2022weighted}, to Appendix~\ref{app:pc0_tsp}.



\begin{figure}[ht]
  \centering
    \includegraphics[scale=0.3, trim=0cm 0.2cm 0cm 0.3cm, clip]{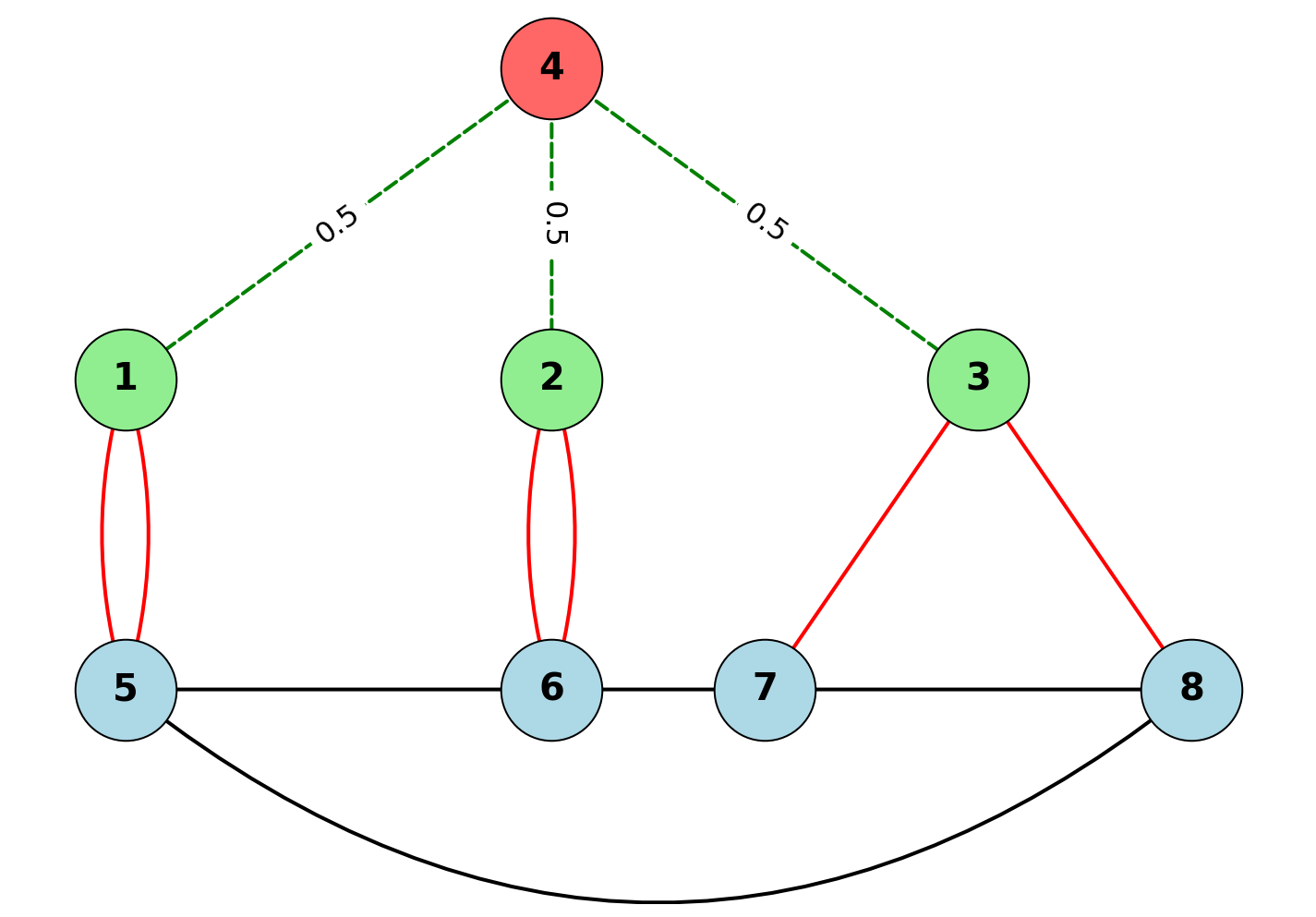}
  \caption{Equation gadget for $\mcst$ found by AlphaEvolve, when reducing from a $\threelin{2}$ equation $\predeqv{1}{2}{3}$ of the form $x\oplus y\oplus z=1$. Vertices $\{1, 2, 3\}$ represent variables in the $\threelin{2}$ equation. The {\color{red} red} edges represent the forced edges. The dashed {\color{Green} green} edges represent the special edges. All {\color{red} red} and black edges have weight one.}
  \label{fig:tsp_AE_pc1}
\end{figure}

\mypar{Core ideas, and the importance of AlphaEvolve} Of the problems discussed in this paper, TSP required the most human involvement, primarily because there was no existing search framework for gadget reductions (analogous to~\cite{trevisan2000gadgets} for $\maxcut{k}$). In particular, the complete reduction in~\cite{chlebik2022weighted} contains a scaffolding beyond the equation gadget, and all components of the reduction are analyzed together as a single object. Within the proof structure of prior work~\cite{chlebik2022weighted,karpinski2015new} it seems unclear how to abstract out a set of verifiable constraints under which AlphaEvolve can be used to search for better equation gadgets alone. In this work, we modularized both the soundness and completeness proofs to depend on well-defined soundness and completeness parameters, defined as optimization problems on the equation gadgets themselves (\Cref{def:coreTSPdef}). This modularization may be of independent interest for future work. We used AlphaEvolve to search for equation gadgets that maximized the final inapproximability ratio obtained from these soundness and completeness parameters.


The optimization problem for the equation gadget search can be cast as a mixed integer program (MIP). Because of its simplicity, it is perceivable that the gadget in~\Cref{fig:tsp_AE_pc1} can be found by directly solving the MIP assuming the number of auxiliary vertices are known in advance. However, because of the large number of constraints involved in the equation gadget corresponding to $x\oplus y\oplus z=0$ (in~\Cref{fig:tsp_AE_pc0}), i.e., around  $11!$, using traditional SMT/MIP solvers would face the same computational bottlenecks described (in~\Cref{sec:othercompapproach}) for the $\maxcut{k}$ problem.

\section{Discussion on AI-assisted Mathematics and Complexity Theory}
\label{sec:discussion}
The area of using AI for Math/CS research is vibrant and changing fast, but we approach it from a particular point of view: we seek to focus on classical and well-studied problems to derive results that stand the test of time. Employing AI currently appears to be the only known way to achieve the results in this paper.
In studying the role of AI in assisting mathematical discovery, we must consider at least these scenarios:
\begin{enumerate}
     \item We invoke a language model to summarize the state of prior art, to chart a research plan towards new theorems, or to directly generate fragments of (or entire) proofs.
     \item We use AI-derived tools such as AlphaEvolve to generate better proof elements (gadgets, graphs).
     \item We use custom code independent of AI to discover better proof elements.
     \item We discover the same or better proof elements by hand.
     \item  A combination of (1)-(4).
\end{enumerate}

\mypar{Literature summaries} The idea that LLMs can usefully generate summaries of prior literature in a field has been mooted about for some time~\cite{wang2023can,wang2023boolean,he2024large,goldberg2024usefulness}, and indeed our experience confirmed this on a number of prompts. The results (modulo occasional hallucination) are a good starting point for deeper exploration and understanding. Interestingly, across many examples we could rapidly generate an overview of the art in an unfamiliar field; we believe this capability will increasingly be used by scientists working across fields so that --- for instance --- an algebraist can come up to speed with Ramsey Theory. We believe that in time, this will lead to more fluent cross-pollination across disciplines.
We have not yet been able to prompt an LLM into providing a usable research plan to obtain new results (such as the ones we report), and this is the focus of deeper efforts e.g., Google's Co-Scientist~\cite{gottweis2025aicoscientistalphaev} program.

\mypar{Direct Prompting} There have been efforts to generate proofs for open mathematical statements via prompting an LLM directly~\cite{vanraamsdonk2025finiteentropysumsquantum, Bubeck2025_GPT5_Proof, orabona2025new,diez2025mathematicalresearchgpt5malliavinstein,jang2025pointconvergencenesterovsaccelerated,bubeck2025early}. This approach has seen mixed success. In some cases an LLM was indeed able to generate the complete and correct proof of a previously unproven statement~\cite{Bubeck2025_GPT5_Proof,bubeck2025early} (although it is possible that a persistent human could perhaps have derived the same), but in many cases~\cite{vanraamsdonk2025finiteentropysumsquantum,orabona2025new,diez2025mathematicalresearchgpt5malliavinstein,jang2025pointconvergencenesterovsaccelerated} the LLM could only generate a proof sketch, which ultimately had to be filled in by humans. At the time of this writing, one significant piece of work in this category is~\cite{bubeck2025early}. This report documents GPT-5's ability to accelerate research in mathematics and theoretical computer science, demonstrating how ``scaffolding'' --- the technique of priming the model with simpler, related warm-up problems --- enabled it to derive improved bounds for online algorithms and formally prove previously open conjectures in graph theory. The authors highlight the model's capacity to go beyond retrieval, successfully constructing novel counterexamples and proof strategies that had previously eluded human experts, such as generating a complex counterexample to the ``Follow-the-Leader'' algorithm in convex body chasing and proposing the ``stability-style analysis'' key to solving Erdős Problem \#848~\cite{bloom_erdos}. In a follow-up to our work,~\cite{woodruff2026accelerating} use Gemini to solve a large class of problems spanning math, physics, and theoretical computer science.

In general, these proofs/sketches currently require a human to verify correctness. Our own~\emph{natural attempts} to prompt a standard LLM into directly generating the kinds of combinatorial structures in this paper met with failure~\footnote{For example, GPT 5.2 and Gemini 3.0 Ultra could not generate the $16/17$-NP-hardness gadget for \maxcuttwo, even though the gadget is explicitly mentioned in~\cite[Figure 4.1]{trevisan2000gadgets} --- the prompt being~\texttt{Can you generate the gadget that gives the construction for $16/17$ approximation of MAX-CUT under NP hardness?}}. While this may become feasible as LLM reasoning capabilities improve~\cite{luong2025advanced}, a formal comparison is beyond the current scope. We emphasize that our constructions \emph{always come with a certificate of correctness}  that is formally verified via standard computational approaches. Once the verification code is sound, there is no further need for human scrutiny.

\mypar{Demonstrating AlphaEvolve's breadth} The recent paper~\cite{georgiev2025mathematical} is a sweeping tour de force of LLM-guided evolutionary search (AlphaEvolve) in autonomously discovering novel mathematical constructions across analysis, combinatorics, and geometry, surprisingly finding a counterexample to the ``Four Guards'' logic puzzle by creatively engaging in ``prompt injection'' against the verifying language model. It notably improves bounds on long-standing problems like the Kakeya needle problem and the Ring Loading Problem, often achieving results competitive with or superior to human-designed benchmarks with minimal problem-specific tuning.
Importantly, in most of this work (including our gadget reductions), AlphaEvolve cannot directly ``evolve the theorem''; rather, its evolutionary process is guided by a synthetic scoring function that guides towards a better theorem.

\mypar{Computational methods in hardness of approximation} In both average case hardness for random graphs, and NP-hardness for approximating $\maxcut{k}$, computer assisted methods~\cite{trevisan2000gadgets,haastad2001some,kunisky2024computational} have been used previously.~\cite{kunisky2024computational} used computational methods to generate $d$-regular Ramanujan graphs (for $d\in\{3,4\}$) with large cut values (or independent sets). While they do not specify their computational approach, we could replicate their results by random sampling of $d$-regular graphs, and testing for the properties by brute force.
For reasons mentioned in~\Cref{sec:avhardness},~\cite{kunisky2024computational} could demonstrate their lower bounds for $n\leq 12$, whereas the graphs we find go up to $n=163$.

For the NP-hardness gadget reduction in~\Cref{sec:maxkcut}, both the soundness and the completeness constraints can be translated into a linear program via skolemization~\cite{trevisan2000gadgets}. However, the size of such a linear program is doubly exponential in the number of vertices in the constraint graph. As a result, {even for $\maxcut{3}$}, running a linear program (LP) with the~\emph{canonical number of variables}~\cite{trevisan2000gadgets} $\approx 3^{3^6}$ becomes infeasible with the standard LP solvers we know of. One can directly attack the non-convex program using SMT/MIP solvers, but with the number of constraints being exponential in the number of variables $n\geq 14$, they too did not seem to scale. (See~\Cref{sec:othercompapproach} for a detailed discussion.) In the case of hardness of approximating the TSP, the soundness and completeness constraints can also be cast as MIP constraints. Even with a modest number of vertices in the equation gadget (e.g., twelve in~\Cref{fig:tsp_AE_pc0}), the number of constraints becomes prohibitively large for standard MIPs to handle (around $11!\approx 3 \times 10^7$).

\mypar{Gadget design by hand} It is conceivable that a human expert or a highly customized computational verifiers could eventually find these solutions. However, their discovery is likely beyond the reach of simple "pencil-and-paper" methods, and standard SMT/MIP solvers failed to produce them. Further, humans intuitively cut the search space through insights into symmetries in the constructed objects; in the case of our TSP gadget, it appears that asymmetry was central to the improvement obtained.

\mypar{AI with significant human effort} Our work and~\cite{tao_erdos_1026} are instances where AI and some significant human effort were deployed in combination. In particular,~\cite{tao_erdos_1026}  chronicles the rapid resolution of Erdős Problem \#1026~\cite{bloom_erdos} through a hybrid workflow where human ingenuity guided multiple AI systems—including AlphaEvolve for generating optimal constructions and the automated theorem prover Aristotle~\cite{achim2025aristotle} for formal verification in Lean. By integrating crowd-sourced mathematical insights with AI-driven deep literature search and computational discovery, the collaboration successfully generated the solution in under 48 hours. In our case, we had to refactor (by hand) the proof logic especially for metric TSP, in order for the problem to be amenable to AlphaEvolve.

In the spirit of Turing's imitation game~\cite{turing1950computing}, one strong test of robustness for AI-assisted mathematical discovery is durability: whether or not new results obtained with AI assistance are superseded by humans (possibly with computer assistance). We note that some results in AI-assisted mathematics have in fact been matched or improved quickly without the use of AI~\cite{gerbicz2025sums,barzilai2025convex}, and in some cases the results existed (even before the problem was posed~\cite{alon2024graph,recht2023cratedigging}) but were not known to the authors when AI methods were applied~\cite{alexeev2025forbidden,bubeck2025early}.

\mypar{Concluding remarks} While our experience here is limited and far from definitive, we believe some early themes are emerging. 


First, language models can generate research plans and summarize the state of the art~\cite{gottweis2025aicoscientistalphaev}. While we have not succeeded in deriving novel results from this, the capability allows non-specialists to quickly learn new domains, which we anticipate will foster greater scientific cross-pollination.

Second, we expect a growing number of proofs of the form ``AI got there first'' without clear evidence of the form ``this couldn't have been done without AI''. In all of these cases (and arguably, across applications of AI to science), we expect verification to be an ongoing bottleneck. We leave it as an open question whether directly prompting an LLM can eventually replicate and surpass our results. 

Third, our work suggests that gadget-based reductions lend themselves to optimization beyond traditional methods (e.g., SMT/MIP solvers), using AlphaEvolve. This in turn suggests that beating AlphaEvolve will generally require non-gadget methods like custom PCPs~\cite{austrin2014new}.

Finally, it is worth dwelling on some failures of our approach. For some problems even if verification is trivial we could not get AlphaEvolve to work.  An example of this is the Hadamard-$668$ conjecture which states that there exists a Hadamard matrix of dimensions $668\times 668$. (More generally, it is conjectured that a Hadamard matrix $H_{n\times n}$ exists for any $n$ which is a multiple of four; 668 is the smallest value for which this is not known.) We attempted to construct one using AlphaEvolve.
Although the search is over a large space of $2^{668^2}$ possibilities, verifying any candidate for correctness only takes $2,23,112$ bitwise multiplications. Even with fast verification, AlphaEvolve was unable to find a construction for $H_{668 \times 668}$. In fact, we failed to get AlphaEvolve to replicate the construction for $H_{428\times 428}$, which was previously the smallest order for which no construction was known until~\cite{kharaghani2005hadamard}; this despite the fact that the construction for 428 is publicly available on the internet.

Looking ahead, it is conceivable that advances in LLM reasoning~\cite{luong2025advanced,gottweis2025aicoscientistalphaev,weil_2025} could be coupled with AlphaEvolve, especially in generating the initial code-snippet and more effective problem specific prompting of the LLMs used by AlphaEvolve. We leave exploration of these problems as future directions. 

\section{Acknowledgments}
We thank Adam Zsolt Wagner for helping us with AlphaEvolve throughout this project, and for his invaluable advice on our experimental setup ---  his insights speeded up our experimentation significantly.
Swarat Chaudhuri not only helped us with our initial work on \maxcut{3}, but also generously advised us throughout on various approaches to verification --- we are deeply grateful to him for this. 
Sushant Sachdeva worked closely with us on \maxcut{3}, TSP, and the design of verifiers and freely shared his valuable intuition on gadget reductions. 
Pasin Manurangsi replicated some of our early gadgets by hand (especially the~\emph{optimal} gadget that reduces from $\predeqv{0}{3}{3}$ to $\predneq$), boosting our confidence early on that AlphaEvolve was progressing in the right direction. Pasin also helped us solidify our understanding of~\cite{karpinski2015new,chlebik2022weighted}. Sidhanth Mohanty made critical contributions in improving the upper bounds on average-case hardness (\Cref{thm:avg-case-ub}). Krishnamurthy (Dj) Dvijotham helped us with the advanced use of MIP solvers, and ablation studies with some of our verifiers. His ablation studies convinced us that direct usage of MIP solvers would not scale to the gadget sizes we are operating with. Shuang Song helped us with numerous engineering challenges. We thank  Mary Chesus, Jonathan Katz, Ravi Kumar, James Manyika, Yossi Matias, Jelani Nelson, Rina Panigrahy, Raluca-Ada Popa, Amit Sahai, Thomas Steinke and Jalaj Upadhyay for their valuable feedback on the manuscript. We thank Uri Feige for pointing out an error in a previous version of Appendix~\ref{sec:avg-appendix-ub}. We  thank Four Flynn and Pushmeet Kohli for their continued support and feedback through the course of this project.

We especially thank Venkat Guruswami and Madhu Sudan for their ongoing encouragement for this work, and for their incisive perspectives on the current status of various results in inapproximability. In particular, Venkat helped us with nuances of the state of the art for $\maxcut{k}$, and pointed us to~\cite{guruswami2009improved,austrin2014new}.

\bibliographystyle{alpha}
\bibliography{reference}
\appendix

\section{AlphaEvolve as a Framework for Combinatorial Discovery}
\label{sec:backgroundAE}

AlphaEvolve~\cite{novikov2025alphaevolve,romera2024mathematical} is an LLM-based code-mutation agent. In the context of this paper, it evolves code snippets that generate combinatorial structures, scoring generated structures by their fitness for the problem at at hand. These evolved code snippets generated novel~\emph{finite combinatorial structures} that improve results in the hardness of approximation. Previously AlphaEvolve has been used in other of scientific domains~\cite{romera2024mathematical,novikov2025alphaevolve}. In the following, we provide a self-contained description of AlphaEvolve. Additionally, we highlight the system enhancements needed to achieve the combinatorial structures in this paper. 

\begin{figure}[htb]
    \centering
    \includegraphics[trim=4cm 6cm 4cm 3cm, scale =0.6]{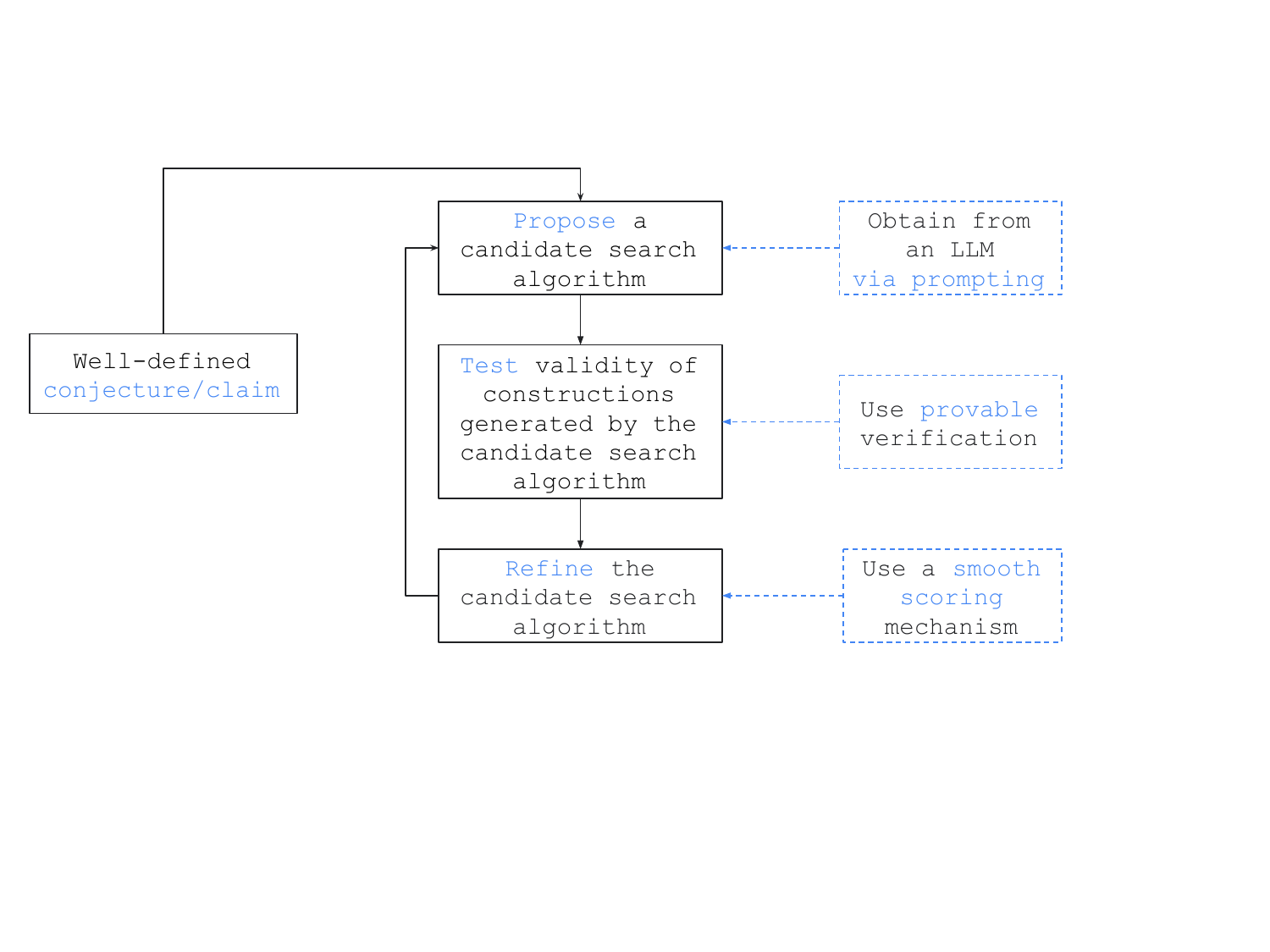}
    \caption{Propose-test-refine (PTR) paradigm: Defining combinatorial search with AlphaEvolve. The solid lines define the control flow, and the dashed lines define comments.}
    \label{fig:PTR}
\end{figure}

\mypar{Propose-test-refine (PTR) paradigm}
We operate in the PTR paradigm~\footnote{The name is motivated from the Propose-test-release (PTR) framework in the differential privacy literature~\cite{dl09}, which performs a similar task of testing a proposed candidate construction for (privacy) properties, before releasing it.} for setting up experiments to discover novel combinatorial structures. While the paradigm is implicit in prior work~\cite{novikov2025alphaevolve}, we first make it explicit, and then instantiate the extensions needed in this paper. The PTR paradigm consists of three components: 
\begin{enumerate}
    \item \emph{Suitably defined conjecture/claim}: The first step towards discovery is a concrete definition of the combinatorial structure being sought. For example, it can be a weighted graph with some fixed set of vertices, or a SAT formula with a fixed number of clauses and variables. In this paper, we restrict these to~\emph{finite structures} rather than a parameterized family~\footnote{Let $\calS$ be a class of combinatorial structures (e.g., graphs, hypergraphs, codes, formulas). A~\emph{parameterized family of structures} from $\calS$ is a sequence $\{S_n\}_{n\in\mathbb{N}}$ s.t. each $S_n\in\calS$ is a finite structure whose size depends on the parameter 
$n$.}. While the source and the target problems form a parameterized family (which enables us to prove theorems with  $\forall n$ quantification), the gadgets used to translate between them are finite structures. (A~\emph{gadget} may be thought of as a combinatorial structure mapping a source constraint to a linear combination of target constraints~\cite{haastad2001some, trevisan2000gadgets}.)

    \item \emph{An initial candidate search algorithm}: AlphaEvolve uses an LLM to make syntactically valid modifications to a code snippet, with the goal of improving a fitness score for the combinatorial structure it generates. Denote the code snippet at time $t$ by $\codesn_t:\calD\to\calS$, where $\calD$ is the domain capturing problem specific parameters and any inputs to the snippet, and $\calS$ is the range space of all possible combinatorial structures. Additionally, let $\score:\calS\to\mathbb{R}$ be the function that assigns a fitness score to a structure. The next function $\codesn_{t+1}$ is chosen as follows:
    \begin{equation}
        \codesn_{t+1} \leftarrow\texttt{AlphaEvolve}\left(\left(\codesn_0(d_0), \score(\codesn_0(d_0))\right),\ldots, \left(\codesn_{t}(d_{t}), \score(\codesn_{t}(d_{t}))\right)\right)
        \label{eq:ae_update}
    \end{equation}
    
    \texttt{AlphaEvolve} makes an LLM call using the tuple of inputs in the prompt to the LLM, prompting it to output the description of $\codesn_{t+1}$, with the goal of increasing the score\footnote{Technically, AlphaEvolve does not feed in all the prior functions as the prompt to the LLM. Instead, it uses a genetic algorithm to select a subset of them while maintaining high scores, and diversity~\cite[Figure 1]{romera2024mathematical}.}. We emphasize the following subtle point: AlphaEvolve does not directly search for combinatorial structures, but instead for~\emph{a code snippet that would output a combinatorial structure with a high score}. \Cref{fig:initcode} shows an initial code-snippet to be evolved by AlphaEvolve.

\begin{figure}
\begin{lstlisting}[style=mypython]
# EVOLVE-BLOCK-START

def gadget_construction() -> nx.graph:
  """Returns a candidate gadget construction"""
  gadget = nx.Graph()
  # Create a complete graph with random weights
  max_scoring_gadget = gadget
  max_score = evaluate_internal(gadget)
  while time.time() - start_time < 1000.0:
    # Make random modifications to gadget, to create
    # a new gadget: new_gadget
    current_score = evaluate_internal(new_gadget)
    if max_score < current_score:
      max_scoring_gadget = new_gadget
      max_score = current_score

  return max_scoring_gadget
 
def evaluate_internal(gadget):
  # Compute the score on the gadget, and return
  return score(gadget)

# EVOLVE-BLOCK-END
\end{lstlisting}
\caption{Initial code-snippet for AlphaEvolve to evolve.}
\label{fig:initcode}
\end{figure}
\item \emph{A verifier with a smooth scoring function}: In the problems we consider, we seek combinatorial structures in spaces of size exponential in the description of the structure. A well-defined verifier that validates (and scores) the constructions generated by the $\codesn_t$'s (in~\Cref{eq:ae_update}) guides this search. 
    
    Concretely, for
    average-case hardness in~\Cref{sec:avhardness}, this corresponds to validating whether the graphs are Ramanujan with certain properties (e.g., having a large cut, or independent set), and for worst-case NP-hardness in~\Cref{sec:maxkcut,sec:tsp}, this corresponds to verifying whether the~\emph{soundness} and~\emph{completeness} constraints (e.g.,~\cite{trevisan2000gadgets,haastad2001some},~\Cref{lem:sound-tsp,lem:completeness}) for the gadget reduction are satisfied.

    As mentioned earlier, every structure generated by AlphaEvolve needs to be scored, including invalid ones that violate certain constraints.  The scoring function  $\score: \calS\to\mathbb{R}$ is used to hill-climb on a optimization landscape over the combinatorial structures. For invalid structures, it decides how far from the constraint boundary the current structure lies. (Informally, a scoring function is~\emph{smooth} if its value gracefully increases/decreases as one moves away from the constraint boundary.) This is where the most creativity and problem-specific domain expertise comes in. In particular, the efficacy of the scoring function largely decides the overall capability of the system to discover novel structures. Also, scoring is usually the~\emph{most compute intensive} step in the PTR paradigm.

    \mypar{Our contributions} Our main contributions to the PTR paradigm are the following: a) using AlphaEvolve itself to accelerate the verification for $\maxcut{k}$ (\emph{sometimes up to $\speedupfour$}), and b) an approach towards using AlphaEvolve for solving optimization problems that involve mixed-integer programming  (MIP) style constraints. For the $k=3$ case, the AlphaEvolve optimized verification code constructed a verifier that offloaded a $O(3^m)$ time computation to a highly optimized tensor contraction operation in~\texttt{numpy}, and only performing slow python-based computation for $O(3^{m/3})$ time. We discussed each of these contributions in detail in~\Cref{sec:search,sec:fasterVerification}.    
\end{enumerate}



\section{Proofs of Average Case Theorems}
\label{app:avg-case}

\subsection{Proof of \Cref{thm:avg-case-lb}}\label{sec:avg-appendix-lb}

Below we specify three Ramanujan graphs $G_4^\MC$, $G_3^\IS$, and $G_4^\IS$ in the sparse6 format~\cite{sparse6}. We also specify a cut $S_4^\MC$ (in the form of a subset of vertices), and independent sets $S_3^\IS$ and $S_4^\IS$ that witness lower bounds on the actual values of $\MC(G_4^\MC)$, $\IS(G_3^{\IS})$, and $\IS(G_4^{\IS})$. The vertices of all graphs are $0$-indexed, as dictated by the sparse6 format. 
We start by defining $G_4^\MC$ and $S_4^\MC$:
\begin{lstlisting}
    >>sparse6<<:~?@{`oQ?bOCHBGIPBHECdOmReom\\A?aECwaAEGKCEiKIGi[NfG?PCwka_AEijomS_qIadPetIJ[DFaOlgqyXKbYAIqolbpi?@@qVJJmx`ryfhGSQDhSfKhonOhOwOYSudgKQ_OoYaPS^e@tGgH]PHhAWMC}ECCaIJS]eIsm@EbQ@@q@OdAgodctRbAwsfQaZNGWnOtQNGryNESzLRxwfJTmA@qasNSeDBcz@QWgQJB}oODDYlUDciRLCTwwVHbAXNeQWJrDfcASuWxOcNR}ELT@Q_rCrVx[eLRYZJc\\]aOh?UH{cOSadKcTMmsX\\YgsdQdiZNstTfBPOVwSHJcIOFspXbcHZWn
\end{lstlisting}
\begin{lstlisting}
     0 2 6 7 13 15 16 17 19 20 21 24 25 27 29 31 33 35 36 37 40 42 47 49 50 54 55 56 58 60 61 62 63 65 67 69 72 73 74 77 78 79 83 84 85 86 89 96 99 101 102 103 104 105 110 111 112 113 114 117 120 121 123
\end{lstlisting}
$G_3^\IS$ and $S_3^\IS$ are defined by the following.
\begin{lstlisting}
    >>sparse6<<:cb?gGMGE_OGo]FBDoggiAGabYCETCESa\\YGHFTGC}aTLObPmcVLqQ|KrKNraaAIXO~\n
\end{lstlisting}
\begin{lstlisting}
    0 15 29 6 4 14 20 28 13 7 12 35 16 30 31 22 32
\end{lstlisting}
Finally, $G_4^{\IS}$ and $S_4^\IS$ are defined by the following.
\begin{lstlisting}
    >>sparse6<<:~?Ab_O?_g@a??@gMa_K_WJ`WNacX?cD?sMbwA`?P_O[cWF@_Q_CF@Kf?Wa`o\\CK[C{]Ck_dgWBS]C_f`OLD?jeG^cS@egdc_ea[x@?u_c{?sVBKXf?}bSDAwhgW@ECQF[B@WZE{LCOk`OWEstFs?B{QgSHE?}gCsHcDB?khkkEG|fKjaGlHSH@XBbWw__o`@S_wkHSRID[@?NEcReXX`OjJkvF`\\JsKAXLIS]DKta_sHHTbGqH{CGPQa?ybW|K{PBgeJcICPCH[``GwFHMbhBI[eHKNBcSEGr_wzJs~J@ne@JkkvK[{J[Ef`Magz`_wJ\\IIXvePWapCgH\\L[VEX?g@AM`tms[BwzIS@BotLs_E`RNtYMPuagbM`vc_eI|gMXwdpGNlbMXtb@LKxs`xVKQGgPPdPallxPINfPE_OUItZLYLQKBNAFaPJNi@dhxP{yNYIQK\\MAUeHOMxyaHOKqMdPFJQ@ip{OIQ`h^NaLfxmN|GLyDRCMKYMP|ZLXpQCEEQSRv\n
\end{lstlisting}
\begin{lstlisting}
     0 1 3 4 7 8 9 10 11 12 13 15 18 24 26 27 28 31 34 40 42 43 44 46 48 53 54 55 56 59 60 61 62 65 75 77 78 82 84 90 91 93 94 98 100 101 103 105 107 109 110 112 114 123 124 125 127 128 129 132 134 135 136 138 141 144 145 146 150 153 154 159 160 162
\end{lstlisting}

The following properties can easily be verified and immediately imply \Cref{thm:avg-case-lb}.
\begin{enumerate}
    \item $G_3^\IS$ is a $3$-regular Ramanujan graph on $36$ vertices. $G_4^\MC$ and $G_4^\IS$ are $4$-regular Ramanujan graphs on $124$ and $163$ vertices respectively.
    \item $S_3^\IS$ and $S_4^\IS$ are independent sets in $G_3^\IS$ and $G_4^\IS$ respectively with $|S_3^\IS| = 17$ and $|S_4^\IS| = 74$. Consequently, $\IS(G_3^\IS)\geq 17/36$, and $\IS(G_4^\IS)\geq 74/163$.
    \item The number of edges in $G_4^\MC$ with exactly one endpoint in $S_4^\MC$ is equal to $226$. Consequently, $\MC(G_4^\MC)\geq 113/124$.
\end{enumerate}

It can be easily verified that all the graphs are Ramanujan proving the claimed lower bounds on $\gamma_4^\MC$, $\gamma_3^\IS$, and $\gamma_4^\IS$. 

\subsection{Proof of \Cref{thm:avg-case-ub}}\label{sec:avg-appendix-ub}

The goal of this section is to prove improved upper bounds on $\sigma_d^{\MC}$ and $\sigma_d^{\IS}$. The same upper bounds apply to $\gamma_d^{\MC}$ and $\gamma_d^{\IS}$, as will be clear in the proof. We will give a detailed argument for the Max-Cut case, and sketch the (minor) changes required for Max-Independent-Set at the end.

As with previous refutation algorithms~\cite{hoffman2003eigenvalues,haemers2021hoffman}, our efficient certificate will be a bound on the \emph{spectral expansion} of a random graph $G\sim \calG(n,d)$, which is defined as $\lambda^*(G)=\max_{i>1}|\lambda_i|$ where $\lambda_1\geq \lambda_2\geq \ldots\geq \lambda_n$ are the eigenvalues of the adjacency matrix $\bfA$ of the graph $G$. Friedman's theorem proves that $\lambda^*$ concentrates around $2\sqrt{d-1}$.

\begin{theorem}[\cite{friedman2008proof}]
    For all $d\in \mathbb{N}$ and $\eps>0$, with probability $1-o_n(1)$ over $G\sim \calG(n,d)$, we have $\lambda^*(G)\leq 2\sqrt{d-1}+\eps$.
\end{theorem}

\begin{defn}[$d$-ary tree]
    The $d$-ary tree of depth $L$, $T_{d,L}$ is a rooted tree, where the root has $d$ children, and all non-root vertices at distance $<L$ from the root have $d-1$ children.
\end{defn}

Our improvement on Hoffman's classical bounds~\cite{hoffman2003eigenvalues,haemers2021hoffman} consists of concluding stronger bounds on $\MC(G)$ and $\IS(G)$ given that $\lambda^*(G)\leq \lambda$ for $\lambda = 2\sqrt{d-1}+\eps$.
Let us denote by $\Omega_{d,\lambda}$ the set of $d$-regular graphs $G$ such that $\lambda^*(G)\leq \lambda$.
We will require the notion of a \emph{labeled} $d$-ary tree, which we define below.  
\begin{defn}[Labeled $d$-ary tree]
    A $\{\pm 1\}$-labeling of $T_{d,L}$ is a mapping $y$ from the vertices of $T_{d,L}$ to $\{\pm 1\}$.
\end{defn}

Let $\Lambda_{d,L}$ be the collection of distributions over $\{\pm 1\}$-labelings of $T_{d,L}$, where the vertices are labeled by elements of $\{\pm 1\}$. For most of the remainder of this section, we will show how we obtain our upper bounds on $\MC(G)$ given $G\in \Omega_{d,\lambda}$. Our upper bound will be parametrized by a nonnegative integer $L\in \mathbb{N}$. Interestingly, we will exactly recover Hoffman's bound~\cite{hoffman2003eigenvalues} when $L=1$. The idea is essentially to write a \emph{finite} LP relaxation of the problem of finding the supremum of the cut fraction ${|\delta_G(S,\bar{S})|}/{|E(G)|}$ over all $G\in \Omega_{d,\lambda}$, by projecting all the information about $G$ and $S$ onto ``local'' neighborhoods at distance $L$.

Let $G\in \Omega_{d,\lambda}$ be an $n$-vertex graph, and let $S\subseteq V(G)$ be a cut in $G$. 
Let $\alpha\in [0,1/2]$ be a value that approximates $|S|/n$, in the sense that $|\alpha - |S|/n|\leq \delta$. Define the cut vector $\bfx\in \{-1,1\}^n$ by $x_i = (-1)^{\mathbf{1}\{i\in S\}}$. We will associate the pair $(G,S)$ with a distribution $\mu_{G,S}\in \Lambda_{d,L}$, which can be sampled from as follows:
\begin{itemize}
    \item Sample a random ordering of the neighboring edges of each vertex in $G$.
    \item Sample a random vertex $i\sim [n]$.
    \item We will label $T_{d,L}$ by the values of $x$ at all length $\leq L$ nonbacktracking walks from $i$ (that is, walks that don't use the same edge twice in a row). Formally, a vertex $v$ of $T_{d,L}$ can be written as a sequence $(q_1, \ldots, q_k)$ for $0\leq k\leq L$, $1\leq q_1\leq d$, and $1\leq q_{l}\leq d-1$ for $l>1$. We will label $v$ by the $x$ value of the final vertex in the corresponding path starting from $i$ in $G$. In particular, let $i_0 = i$, and for $1\leq l\leq k$ let $i_l$ be the $q_l^{th}$ neighbor of $i_{l-1}$ in the sampled ordering, excluding $i_{l-2}$ if $l>1$. We will label $v$ by $y_v:=x_{i_{k}}$.
    \item Output the resulting labeling $y$ of $T_{d,L}$.
\end{itemize}

We will deduce some linear constraints on $\mu_{G,S}$, in terms of its probability distribution function.

\begin{lem}[Local Consistency Constraints]\label{lem:mc-1}
    Let $r_0$ be the root of $T_{d,L}$, and let $r_1$ be a child of $r_0$. For $i\in \{0,1\}$, let $T_i$ be the subtree of $T_{d,L}$ rooted at $r_i$ containing all vertices at distance at most $L$ from $r_{1-i}$. Then, $T_0$ is isomorphic to $T_1$, and the marginal distribution of $\mu_{G,S}$ on $T_0$ is identical (up to the isomorphism) of that on $T_1$.
\end{lem}
\begin{proof}
    The marginal distribution on $T_0$ can be viewed as the following: sample a random vertex $i$ in $G$ and a random neighbor $j\sim i$. Output the labels $x$ at all endpoints of a nonbacktracking walk from $i$ of length at most $L$ if $(i,j)$ is the first edge in the walk, and length at most $L-1$ otherwise.

    The marginal distribution on $T_1$ is identical, with the roles of $i$ and $j$ switched. Since the distribution of $(i,j)$ is identical to that of $(j,i)$, both marginal distributions are identical.
\end{proof}

\begin{lem}[Average Value of root]\label{lem:mc-2}
    We have $2\alpha-1-2\delta\leq \mathbb{E}_{y\sim \mu_{G,S}}[y_{0}]\leq 2\alpha-1+2\delta$.
\end{lem}
\begin{proof}
    We can directly compute $\mathbb{E}_{y\sim \mu_{G,S}}[y_0] = \mathbb{E}_{i\sim [n]}[x_i] = 2|S|/n-1$. The sandwiching inequalities follow from the fact that $|\alpha-|S|/n|\leq \delta$.
\end{proof}

\begin{lem}[Spectral Constraints]\label{lem:mc-3}
    The bound $\lambda^*(G)\leq \lambda$ implies the bounds \[(2\alpha-1-2\delta)^2 \leq \mathbb{E}_{y\sim \mu_{G,S},\ell}[y_{0}\cdot y_\ell]\leq \left(2\alpha-1 + 2\delta\right) + (\lambda/d)^L\]
    if $L$ is even, and 
    \[(2\alpha-1-2\delta)^2 - (\lambda/d)^L \leq \mathbb{E}_{y\sim \mu_{G,S},\ell}[y_{0}\cdot y_\ell]\leq \left(2\alpha-1 + 2\delta\right)^2 + (\lambda/d)^L\]
    if $L$ is odd. Here $0$ denotes the root, and $\ell$ denotes the (random) endpoint of an $L$-step random walk in $T_{d,L}$ starting from the root.
\end{lem}
\begin{proof}
    We begin by computing 
    \[\mathbb{E}_{y\sim \mu_{G,S},\ell}[y_{0}\cdot y_\ell] = \frac{\bfx^\top(\bfA/d)^L\bfx}{n} = \frac{\bfx^\top \bfA^L \bfx}{d^L\cdot n}.\]
    Let us write $\bfA = d \frac{\mathbf{1}\mathbf{1}^\top}{n} + \tilde{\bfA}$, where $d=\lambda_1$ is the trivial eigenvalue of $\bfA$ and $\tilde{\bfA}$ is the portion of $\bfA$ orthogonal to the trivial eigenvector $\mathbf{1}$. That is, $\tilde{A}\cdot\mathbf{1} = \mathbf{0}$. Note that $\|\tilde{\bfA}\|_{\sf op}\leq \lambda$. Let us proceed with our computation.
    \begin{align*}
        \mathbb{E}_{y\sim \mu_{G,S},\ell}[y_{0}\cdot y_\ell] &=  \frac{d^L\cdot \bfx^\top\mathbf{1}\mathbf{1}^T \bfx}{d^L\cdot n^2} + \frac{\bfx^\top \tilde{\bfA}^L \bfx}{d^L\cdot n}\\
        &= \frac{(2|S|-n)^2}{n^2} + \frac{\bfx^\top \tilde{\bfA}^L \bfx}{d^L\cdot n}.
    \end{align*}

    The Lemma follows from (1) the fact that $|\alpha - |S|/n|\leq \delta$, and (2) the bounds $|\bfx^\top \tilde{\bfA}^L \bfx|\leq \|\tilde{\bfA}\|_{\sf op}\cdot \|\bfx\|_2^2\leq \lambda n$, along with the bound $\bfx^\top \tilde{A}^L \bfx\geq 0$ if $L$ is even.
\end{proof}

\begin{lem}[Objective Value]\label{lem:mc-4}
    The cut fraction ${|\delta_G(S,\bar{S})|}/{|E(G)|}$ equals $(1-\mathbb{E}_{y\sim \mu_{G,S},i\sim [d]}[y_{0}\cdot y_i])/{2}$, where $0$ denotes the root, and $i$ denotes the $i^{th}$ neighbor of the root of the $d$-regular tree of depth $L$.
\end{lem}
\begin{proof}
    This is again a direct computation. We can write
    \[\mathbb{E}_{y\sim \mu_{G,S},i\sim [d]}[y_{0}\cdot y_i] = \frac{\bfx^\top \bfA \bfx}{nd} = 1-2\frac{|\delta_G(S,\bar{S})|}{|E(G)|},\]
    and rearranging gives the identity.
\end{proof}

\Cref{lem:mc-1,lem:mc-2,lem:mc-3} are all linear constraints in the pdf of the distribution $\mu=\mu_{G,S}$, and \Cref{lem:mc-4} is a linear objective function. We may solve this LP in the variables $\mu$ for small values of $L$ and $d$ to obtain valid upper bounds on objective function, the cut value $|\delta_G(S,\bar{S})|/|E(G)|$. If we solve this LP for a collection of discrete values $\alpha\in \{0,\delta, 2\delta, \ldots, 1\}$, we obtain an upper bound on the cut value of \emph{any} cut of a graph $G\in \Omega_{d,\lambda}$.

\paragraph{Computational Efficiency.} The above LP appears intractable -- it has $2^{46}$ variables even for $d=3$, $L=4$. We apply two optimizations to make it more tractable in practice.
\begin{enumerate}
    \item Use the same variable for isomorphic labelings of $T_{d,L}$. In particular, if there is a graph isomorphism of $T_{d,L}$ that maps a labeling $y$ to another labeling $y'$, then \Cref{lem:mc-1} implies that they should receive the same probability in $\mu$.
    \item Only consider ``locally optimal'' labelings of $T_{d,L}$. To do this, we apply a preprocessing step, where we ensure that the cut $S$ is locally optimal in the sense that for any vertex $i\in S$, the labeling of $T_{d,L}$ generated by restricting $\bfx\in \{\pm 1\}^n$ to the $L$-step neighborhood of $i$ corresponds to a maximum cut of $T_{d,L}$, subject to the ``boundary conditions'' of labels of the leaves of $T_{d,L}$. Therefore, we can also discard any variables corresponding to such suboptimal labelings of $T_{d,L}$. 
\end{enumerate}

Recall that $\lambda = 2\sqrt{d-1}+\eps$ for arbitrarily small $\eps$. We will now pick $\eps=10^{-5}$. Using the above optimizations, we are able to solve the $1/\delta$ LPs for $d=3,L=4$ (with $4396$ variables), and $d=4, L=2$ (with $88$ variables) for $\delta = 0.0005$, resulting in the universal upper bounds $\MC(G)\leq 0.953$ for $G\in \Omega_{3,\lambda}$ and $\MC(G)\leq 0.916$ for $G\in \Omega_{4,\lambda}$.

An interesting point is that the $d=4, L=3$ case has $11880$ variables and is technically within reach of current LP solvers, but resulted in the \emph{same} bound on $\sigma^{\MC}_4$ as the $d=4, L=2$ case.

\paragraph{Upper bounds on $\IS(G)$.} We conclude with some comments on the modifications needed to obtain upper bounds on $\sigma_d^\IS$. The main modification is that we need to restrict ourselves to distributions in $\Lambda_{d,L}$ that are supported only on $\{\pm 1\}$-labelings that correspond to actual Independent Sets of $T_{d,L}$. That is, we only consider labelings $y$ such that for any edge $\{i,j\}$ in $T_{d,L}$, $y_i$ and $y_j$ are not both equal to $1$. The objective function disappears, and we now only need to find the largest possible value $\alpha$ such that the LP is still feasible. The $1/\delta$ LPs are tractable for $d=3, L=4$ (with $1771$ variables) and $d=4, L=3$ (with $8855$ variables), and they result in the bounds $\sigma_3^\IS\leq 0.476$ and $\sigma_4^\IS\leq 0.457$.
\section{Proofs of \Cref{thm:max-3-cut,thm:max-4-cut}}\label{sec:gadget-appendix}

{
In this section, we give the details of the NP-hardness of approximation proofs for \maxcut{k} for $k\in \{3,4\}$. For ease of exposition, we begin by reintroducing all predicates needed for the proof. All predicates will be over the alphabet $\alphabet$.
\begin{itemize}
    \item $\predneq$ denotes a binary predicate defined by $\predneq(x,y) = \mathbf{1}\{x\neq y\}$.
    \item For $i\in \{0,\ldots, k-1\}$, $\predeqv{i}{k}{3}$ denotes a $3$-ary predicate defined by $\predeqv{i}{k}{3}(x,y,z) = \mathbf{1}\{x+y+z\equiv i\pmod k\}$.
\end{itemize}
Recall that \maxcut{k} is $\csp(\predneq)$, and $\threelin{k}$ is $\csp\left(\predeqv{0}{k}{3},\ldots, \predeqv{(k-1)}{k}{3}\right)$. We now formally define gadgets, specifically in the context of the reduction from $\threelin{k}$ to $\maxcut{k}$. We will need $k$ separate gadgets. In particular, for each defining predicate $\predeqv{i}{k}{3}$ of $\threelin{k}$, we will need an $i$-gadget, which we define below.
\begin{defn} [Gadgets]
\label{defn:gadget}
{
Let $k\geq 2$. 
\begin{itemize}
  \item {\bf $i$-gadget}: For every $i\in \alphabet$, 
  we introduce a gadget $\insti{i}{k}{3}$ that reduces the predicate $\predeqv{i}{k}{3}$ to an instance of $\maxcut{k}$. Formally, $\insti{i}{k}{3}$ is an instance of $\maxcut{k}$ over $3+k+\naux$ variables, where $\naux\in \mathbb{N}$.
  \item {\bf Primary, global, and auxiliary variables:} We partition the set of variables $[3+k+\naux]$ into three \emph{primary variables} $\{1,2,3\}$, $k$ \emph{global variables} $\{4,5,\ldots, k+3\}$, and $\naux$ \emph{auxiliary variables} $\{k+4,\ldots, 3+k+\naux\}$. We will always write an assignment to the variables as a tuple $(\bfx,\bfy,\bfz)\in \alphabet^{3+k+\naux}$ where $\bfx\in \alphabet^3$, $\bfy\in \alphabet^k$, and $\bfz\in \alphabet^{\naux}$.
\end{itemize} }
\end{defn}
    The variables $\bfy$ will be restricted in a way to break the symmetries of $\maxcut{k}$. Concretely, let $Y$ be the set of strings $\bfy$ in $\alphabet^k$ that are lexicographically not larger than any string $\bfy^\sigma=(\sigma(y_1),\sigma(y_2),\ldots, \sigma(y_k))$ that is obtained by permuting the alphabet $\alphabet$ by some permutation $\sigma:\alphabet\to\alphabet$. In particular, note that the vector $(0,1,2,\ldots, k-1)$ is in $Y$. We will only ever assign $\bfy$ to values in $Y$.

\begin{defn}[Soundness and completeness]\label{def:gadget-params}
    We will set $C_i = \left(\predeqv{i}{k}{3}\right)^{-1}(1)$ and $S_i = \left(\predeqv{i}{k}{3}\right)^{-1}(0)$ to be the set of satisfying and unsatisfying assignments for the source predicate $\predeqv{i}{k}{3}$ respectively.

    We will associate four real-valued parameters to the gadget $\insti{i}{k}{3} = \sum_{j\in [m]}\clause_j$ as follows.
    \begin{itemize}
        \item ({\bf Completeness analysis}) $\complete(\insti{i}{k}{3})$ is defined as \[\complete(\insti{i}{k}{3})=\min_{\bfx\in C_i}\max_{\bfz\in \alphabet^{\naux}}\insti{i}{k}{3}(\bfx, \bfy, \bfz)\] for $\bfy=(0,1,\ldots, k-1)$. For $\bfx\in C_i$, we say that $\arg\max_{\bfz\in \alphabet^{\naux}} \insti{i}{k}{3}(\bfx, \bfy, \bfz)$ is the \emph{witness} for $\bfx$. 
        \item ({\bf Soundness analysis}) $\sound_1(\insti{i}{k}{3})$ is defined as \[\sound_1(\insti{i}{k}{3})=\max_{\bfx\in C_i,\bfz\in \alphabet^{\naux}, \bfy\in Y}\insti{i}{k}{3}(\bfx,\bfy,\bfz).\] 
        Also, $\sound_2(\insti{i}{k}{3})$ is defined as
        \[\sound_2(\insti{i}{k}{3})=\max_{\bfx\in S_i,\bfz\in \alphabet^{\naux}, \bfy\in Y}\insti{i}{k}{3}(\bfx,\bfy,\bfz).\] 
        \item ({\bf Number of clauses}) $\countcl(\insti{i}{k}{3})$ is defined to be the number of clauses present in $\insti{i}{k}{3}$.
    \end{itemize}
    \label{def:soundComplete}
\end{defn}
}

We will show that any collection of $i$-gadgets for all $i\in \alphabet$ provides a reduction whose performance depends on the various parameters defined above. After this, in Appendix~\ref{sec:gadget-description}, we give an explicit description of our gadgets and state the parameters achieved by them.

The starting point of our reduction will be the following theorem due to Håstad.

\begin{theorem}[\cite{haastad2001some}]\label{thm:hastad-3lin}
    {
    Let $k\geq 2$. Given an instance $\inst=\sum_{j\in [m]}\clause_j$ of $\threelin{k}$ with $m/k$ clauses from each of the predicates $\predeqv{0}{k}{3},\predeqv{1}{k}{3},\ldots, \predeqv{(k-1)}{k}{3}$, it is NP-hard to distinguish between the following cases for any $\eps>0$:
    \begin{enumerate}
        \item {\bf Completeness case}. $\max_{\bfx}\inst(\bfx)\geq (1-\eps)\cdot m$, or
        \item {\bf Soundness case}. $\max_\bfx \inst(\bfx)\leq (1/k+\eps)\cdot m$.
    \end{enumerate}
    }
\end{theorem}

\Cref{thm:reduction-general} details how our gadgets imply NP-hardness of approximation for $\maxcut{k}$, with $k\in\{3, 4\}$.

\begin{theorem}\label{thm:reduction-general}
{
Let $k\geq 2$. Let $\insti{i}{k}{3}$ be an $i$-gadget that reduces the predicate $\predeqv{i}{k}{3}$ to an instance of $\maxcut{k}$ for each $i\in \alphabet$. For any $\eps>0$, given an instance $\instj$ of $\maxcut{k}$ with $M$ constraints, it is NP-hard to distinguish between the following two cases for any $\eps>0$:
    \begin{enumerate}
        \item {\bf Completeness case}: $\max_{\bfx}\instj(\bfx)\geq (a-\eps)\cdot M$, or
        \item {\bf Soundness case}: $\max_\bfx \instj(\bfx)\leq (b+\eps)\cdot M$,
    \end{enumerate}
    where
    \[a = \frac{\sum_{i\in \alphabet}\complete(\insti{i}{k}{3})}{\sum_{i\in \alphabet}\countcl(\insti{i}{k}{3})},\quad b=\frac{\sound_1(\insti{r_0}{k}{3}) + \sum_{i\in \alphabet\setminus \{0\}}\sound_2(\insti{r_i}{k}{3})}{\sum_{i\in \alphabet}\countcl(\insti{i}{k}{3})}.\]
    Here $r_0,r_1,\ldots r_{k-1}$ is a permutation of $\{0,1,\ldots, k-1\}$ such that $\sound_1(\insti{r_i}{k}{3}) - \sound_2(\insti{r_i}{k}{3})$ is maximized by $r_0$.
}
\end{theorem}

\begin{proof}[Proof of \Cref{thm:reduction-general}]
    {
        We will perform a reduction starting from the NP-hard problem in $\Cref{thm:hastad-3lin}$ to $\maxcut{k}$. Let $\inst = \sum_{j\in [m]}\clause_j$ be an $n$-variable instance of $\threelin{k}$, and suppose that it has $m/k$ clauses using $\predeqv{i}{k}{3}$ for all $i\in \alphabet$. We will show how to reduce $\inst$ to an instance $\instj$ of $\maxcut{k}$.

    Let $\naux$ be the number of auxiliary variables in all gadgets $\{\insti{i}{k}{3}:i\in\alphabet\}$; we can assume they all have the same number of auxiliary variables by introducing new dummy variables to some of them. $J$ will be defined as an instance on $n+k + \naux\cdot m$ variables, which we will denote as $\bfx\in \alphabet^n$, $\bfy\in\alphabet^k$, and $\bfz\in \alphabet^{\naux\cdot m}$. For $j\in [m]$ we will denote the $j^{th}$ block of size $\naux$ in $z$ as $z^j$. We will refer to $\bfx,\bfy,\bfz$ as the primary, global, and auxiliary variables respectively.

    Let $\clause_j=\predeqv{i}{k}{3}(x_{j_1}, x_{j_2}, x_{j_3})$ be a clause corresponding to the predicate $\predeqv{i}{k}{3}$ for some $i$. We will add a copy of $\insti{i}{k}{3}$ on the variables $(\bfx', \bfy, \bfz^j)$, where $\bfx' = (x_{j_1}, x_{j_2}, x_{j_3})$.
    Note that the number of clauses in $\instj$ is $M=(m/k)\cdot \sum_{i\in \alphabet}\countcl(\insti{i}{k}{3})$.

    Now we will invoke \Cref{thm:hastad-3lin}, which says that it is NP-hard to distinguish between the cases that $\inst$ is almost completely satisfiable, and $\inst$ is essentially $1/k$-satisfiable.

    \begin{enumerate}
        \item {\bf Completeness case}: Assuming there is an $\bfx\in \alphabet^{n}$ such that $\inst(\bfx)\geq (1-\eps)\cdot m$, we show how to construct an assignment $(\bfx,\bfy,\bfz)\in \alphabet^{n+k+\naux\cdot m}$ such that $\instj(\bfx,\bfy,\bfz)\geq a\cdot M$. We will set $\bfy = (0,1,2,\ldots, k-1)$, and for each clause $\clause_j =\predeqv{i}{k}{3}(x_{j_1},x_{j_2},x_{j_3})$, $\bfz^j$ will be set as the witness\\ $\bfz^j=\arg\max_{\bfz\in \alphabet^{\naux}}\insti{i}{k}{3}(\bfx', \bfy, \bfz)$, where $\bfx'=(x_{j_1},x_{j_2},x_{j_3})$.

        It is immediate that $(\bfx,\bfy,\bfz)$ satisfies at least \[\sum_{i\in \alphabet}\complete(\insti{i}{k}{3})\cdot m/k - O(\eps \cdot m) = a\cdot M - O(\eps)\cdot m \geq (a-\eps')\cdot M\] clauses of $J$, where $\eps' = O(\eps)$ can be chosen to be arbitrarily close to $0$.
        
        \item {\bf Soundness case}: We must show that if $\max_\bfx \inst(\bfx)\leq (1/k+\eps)\cdot  m$, then $\max_{\bfx,\bfy,\bfz}\instj(\bfx,\bfy,\bfz)\leq (b+\eps')\cdot M$. Let $(\bfx,\bfy,\bfz)$ be an assignment to the variables of $\instj$.

        The first step we perform is to assume without loss of generality that $\bfy\in Y$. To do this, note that by definition of $Y$, there is some permutation $\sigma:\alphabet\to\alphabet$ such that $\bfy^\sigma = (\sigma(y_1),\ldots, \sigma(y_k))\in Y$. We will apply $\sigma$ to all variables in the full assignment $(\bfx,\bfy,\bfz)$, resulting in no change in the value of $\instj(\bfx,\bfy,\bfz)$ while guaranteeing that $\bfy\in Y$.
        
        For $i\in \alphabet$, let $\alpha_i$ denote the fraction of $\predeqv{i}{k}{3}$ constraints of $I$ that are satisfied by $\bfx$. By assumption, we have $\sum_{i\in \alphabet}\alpha_i\leq 1+k\eps$. Let us define $\hat{\alpha_i} = \alpha_i / (\sum_{i'\in \alphabet}\alpha_{i'})$, so $\sum_{i\in \alphabet}\hat{\alpha_i} = 1$. In other words, $\{\hat{\alpha_i}:i\in \alphabet\}$ defines a probability distribution.

        On the other hand, since $y\in Y$ we can now bound $\instj(\bfx,\bfy,\bfz)$ as follows.
        \begin{align*}
            \instj(\bfx,\bfy,\bfz)&\leq m\cdot\left(\sum_{i\in \alphabet}\alpha_i\cdot \sound_1(\insti{i}{k}{3}) + (1-\alpha_i)\cdot \sound_2(\insti{i}{k}{3})\right)\\
            &\leq m\cdot\left(\sum_{i\in \alphabet}\hat{\alpha_i}\cdot \sound_1(\insti{i}{k}{3}) + (1-\hat{\alpha_i})\cdot \sound_2(\insti{i}{k}{3})\right) + O(\eps)\cdot m\tag{$\sum_{i\in \alphabet} \alpha_i\leq 1+k\eps$}\\
            & = m\cdot\left(\sum_{i\in \alphabet}\hat{\alpha_i}\cdot (\sound_1(\insti{i}{k}{3})  - \sound_2(\insti{i}{k}{3}))+  \sound_2(\insti{i}{k}{3})\right) + O(\eps)\cdot m.
        \end{align*}
        Note that subject to $\hat{\alpha_i}$ being a probability distribution, this is maximized when $\hat{\alpha_i}$ is a point mass on a maximizer of $\sound_1(\insti{i}{k}{3}) - \sound_2(\insti{i}{k}{3})$. Therefore, this is bounded by
        \[m\cdot \left(\sound_1(\insti{r_0}) + \sum_{i\in \alphabet\setminus \{0\}}\sound_2(\insti{r_i}{k}{3})\right) + O(\eps)\cdot m = (b + O(\eps))\cdot M = (b+\eps')\cdot M,\]
        where $\eps'$ can be chosen to be arbitrarily close to $0$. This completes the proof of the soundness case.
    \end{enumerate}
}

\end{proof}

\subsection{Description of our gadgets}\label{sec:gadget-description}

For $k\in \{3,4\}$, we used AlphaEvolve to find gadgets $\{\insti{i}{k}{3}:i\in \alphabet\}$, scoring a particular gadget by the inverse inapproximability ratio $a/b$, where $a$ and $b$ are defined in \Cref{thm:reduction-general}. We now concretely define the specific gadgets found by AlphaEvolve and list the parameters they achieve (\Cref{def:gadget-params}). We derive the final inapproximability results (\Cref{thm:max-3-cut,thm:max-4-cut}) by applying \Cref{thm:reduction-general} on these gadgets.

We have formatted a gadget as a weighted ``edge list'' corresponding to a collection of $\predneq$ clauses. The format is meant to be understood as follows: a tuple $(a,b,w)$ corresponds to $w$ parallel edges from variable $a$ to variable $b$, that is, the gadget contains $w$ copies of the term $\predneq$ applied to variables $a$ and $b$. 

\paragraph{Gadgets for \maxcut{3}.}

We write the list of edges for $\insti{0}{3}{3}$ first, which has $12$ total variables, and hence $6$ auxiliary variables. Note that the global variables $\{4,5,6\}$ are not used in this gadget. See~\Cref{subfig:gadget-zero} for a visual representation of the gadget.
\begin{lstlisting}
[(1, 7, 1), (1, 8, 1), (1, 9, 1), (1, 10, 1), (2, 9, 1), (2, 10, 1), (2, 11, 1), (2, 12, 1), (3, 7, 1), (3, 8, 1), (3, 11, 1), (3, 12, 1), (7, 12, 1), (7, 10, 1), (8, 9, 1), (8, 11, 1), (9, 12, 1), (10, 11, 1)]
\end{lstlisting}

Below (and in~\Cref{subfig:gadget-1}), we write the list of edges for our $\insti{1}{3}{3}$ gadget, which has $6+\naux = 14$ total variables, and hence $8$ auxiliary variables. Note that repeated edges correspond to multiple $\predneq$ clauses on the same pair of variables.

\begin{lstlisting}
[(3, 7, 1), (4, 6, 3), (4, 12, 1), (3, 13, 2), (5, 10, 1), (9, 11, 1), (11, 14, 1), (1, 3, 2), (1, 9, 1), (2, 8, 1), (2, 14, 1), (13, 14, 2), (6, 11, 1), (4, 5, 3), (3, 9, 1), (5, 6, 3), (4, 8, 2), (5, 9, 2), (1, 2, 2), (2, 7, 1), (10, 14, 1), (1, 8, 1), (1, 14, 1), (2, 13, 2), (6, 7, 2), (7, 12, 1), (4, 7, 1), (12, 14, 1), (3, 8, 1), (3, 14, 1), (5, 8, 1), (8, 10, 1), (2, 3, 2), (2, 9, 1), (1, 7, 1), (1, 13, 2), (6, 9, 1)]
\end{lstlisting}

$\insti{2}{3}{3}$ is defined identically to $\insti{1}{3}{3}$, except with the global variables reordered from $(4,5,6)$ to $(4,6,5)$. We can calculate the relevant parameters of these gadgets as below.


\begin{lem}\label{lem:gadget-numbers}
    Let $\insti{0}{3}{3}, \insti{1}{3}{3}$, and $\insti{2}{3}{3}$ be as above. $\insti{0}{3}{3}$ is a $0$-gadget  that maps the predicate $\predeqv{0}{3}{3}$ to an instance of $\maxcut{3}$, with parameters $\complete(\insti{0}{3}{3})=\sound_1(\insti{0}{3}{3})=t(\insti{0}{3}{3}) = 18$, and $\sound_2(\insti{0}{3}{3})=16$. For $i\in \{1,2\}$, $\insti{i}{3}{3}$ is an $i$-gadget with parameters $\complete(\insti{i}{3}{3}) = \sound_1(\insti{i}{3}{3}) = 48$, $\sound_2(\insti{i}{3}{3}) = 46$, and $\countcl(\insti{i}{3}{3}) = 53$.
\end{lem}

Combined with \Cref{thm:reduction-general}, we immediately get \Cref{thm:max-3-cut}.

\paragraph{Gadgets for \maxcut{4}.}

Below we write the weighted edge list for our $\insti{0}{4}{3}$ gadget for $\maxcut{4}$.

\begin{lstlisting}
[(1, 2, 866), (1, 3, 865), (1, 5, 324), (1, 6, 168), (1, 7, 324), (1, 8, 178), (1, 9, 1361), (1, 10, 648), (1, 11, 628), (1, 12, 731), (1, 13, 36), (1, 14, 331), (1, 15, 1038), (1, 16, 473), (1, 17, 73), (1, 18, 442), (1, 19, 1013), (2, 3, 866), (2, 5, 323), (2, 6, 168), (2, 7, 323), (2, 8, 1243), (2, 9, 1361), (2, 10, 724), (2, 11, 261), (2, 12, 731), (2, 13, 36), (2, 14, 331), (2, 15, 65), (2, 16, 218), (2, 17, 601), (2, 18, 463), (2, 19, 995), (3, 5, 324), (3, 6, 168), (3, 7, 322), (3, 8, 1125), (3, 9, 1360), (3, 10, 288), (3, 11, 719), (3, 12, 731), (3, 13, 37), (3, 14, 331), (3, 15, 1037), (3, 16, 509), (3, 17, 603), (3, 18, 149), (3, 19, 66), (4, 5, 1261), (4, 6, 1089), (4, 7, 1259), (4, 9, 173), (4, 10, 19), (4, 11, 11), (4, 12, 167), (4, 13, 1255), (4, 14, 12), (4, 16, 11), (4, 17, 9), (4, 18, 29), (5, 6, 660), (5, 7, 1429), (5, 8, 87), (5, 9, 314), (5, 10, 34), (5, 11, 10), (5, 12, 581), (5, 13, 2), (5, 14, 389), (5, 15, 8), (5, 16, 15), (5, 17, 24), (5, 18, 7), (5, 19, 16), (6, 7, 656), (6, 8, 380), (6, 10, 181), (6, 11, 159), (6, 13, 49), (6, 15, 350), (6, 16, 78), (6, 17, 196), (6, 18, 79), (6, 19, 322), (7, 8, 94), (7, 9, 319), (7, 10, 35), (7, 11, 20), (7, 12, 655), (7, 13, 1), (7, 14, 316), (7, 15, 16), (7, 16, 2), (7, 17, 23), (7, 18, 7), (7, 19, 22), (8, 9, 52), (8, 10, 338), (8, 11, 460), (8, 12, 508), (8, 13, 20), (8, 14, 338), (8, 16, 326), (8, 17, 59), (8, 18, 388), (9, 10, 361), (9, 11, 343), (9, 12, 263), (9, 13, 1265), (9, 14, 60), (9, 15, 25), (9, 16, 288), (9, 17, 199), (9, 18, 207), (9, 19, 25), (10, 11, 3), (10, 12, 344), (10, 13, 226), (10, 14, 97), (10, 15, 404), (10, 16, 3), (10, 17, 3), (10, 18, 103), (10, 19, 82), (11, 12, 296), (11, 13, 360), (11, 14, 160), (11, 15, 60), (11, 16, 109), (11, 17, 2), (11, 19, 436), (12, 13, 898), (12, 15, 517), (12, 16, 120), (12, 17, 264), (12, 18, 238), (12, 19, 448), (13, 14, 344), (13, 15, 34), (13, 16, 262), (13, 17, 300), (13, 18, 178), (13, 19, 6), (14, 15, 190), (14, 16, 218), (14, 17, 134), (14, 18, 66), (14, 19, 242), (15, 16, 97), (15, 17, 328), (15, 18, 352), (16, 17, 2), (16, 19, 252), (17, 19, 357), (18, 19, 73)]
\end{lstlisting}

For $i\geq 1$, our $\insti{i}{4}{3}$ gadget is obtained from our $\insti{(i-1)}{4}{3}$ gadget by reordering the global variables from $(4,5,6,7)$ to $(7,4,5,6)$. The below Lemma can be verified computationally using~\Cref{thm:hastad-3lin}.

\begin{lem}\label{lem:gadget-numbers-four}
    Let $\insti{0}{4}{3}, \insti{1}{4}{3}$, $\insti{2}{4}{3}$, and $\insti{3}{4}{3}$ be as above. For each $i$, $\insti{i}{4}{3}$ is a $i$-gadget  that maps the predicate $\predeqv{i}{4}{3}$ to an instance of $\maxcut{4}$, with parameters
    $\complete(\insti{i}{4}{3})=49535$, $\sound_1(\insti{i}{4}{3})=49538$, $\sound_2(\insti{i}{4}{3}) = 48681$, and $\countcl(\insti{i}{4}{3})=52941$.
\end{lem}

Together with \Cref{thm:reduction-general}, we immediately get \Cref{thm:max-4-cut}.

Finally, for the convenience of the reader, we provide a code snippet that computes the set $Y$ (from~\Cref{defn:gadget}) of allowed assignments to the global variables as a function of $k$ below.
    \begin{lstlisting}[style=mypython]
# Outputs list of allowed assignments to global variables for the 3lin(k) to max-k-cut reduction
def allowed_global_assignments(k):
  # iterate over all k-tuples of elements of Z_k
  ret = []
  for y in itertools.product(range(k), repeat=k):
    # To determine whether y is lexicographically minimal with respect to permuting each entry by the same permutation of Z_k, we check that ith distinct element of y is equal to i-1.
    elements = []
    for element in y:
      if element not in elements:
        elements.append(element)
    if elements == list(range(len(elements))):
      ret.append(y)
  return ret
\end{lstlisting}

\section{Proof of the Hardness of Approximation for TSP}
\label{app:TSP}

In this section we prove~\Cref{thm:tsp}. We will begin by recalling the various predicates we consider in this section, all of which are over the boolean alphabet $\ZZ_2 = \{0,1\}$. Following the convention in the rest of the paper, we define the predicates relevant to this section.
\begin{itemize}
    \item $\predeq$ denotes a binary predicate defined by $\predeq(x,y) = \mathbf{1}\{x=y\}$.
    \item $\predneq$ denotes a binary predicate defined by $\predneq(x,y) = \mathbf{1}\{x\neq y\}$.
    \item $\predeqv{1}{2}{3}$ denotes a $3$-ary predicate defined by $\predeqv{1}{2}{3}(x,y,z) = \mathbf{1}\{x + y + z \equiv 1\pmod 2\}$.
\end{itemize}

The main technical result of this section is~\Cref{thm:tsp2}, which is our generalized hardness statement that is formulated in terms of the soundness and completeness parameters $\sound(H)$ and $\complete(H)$ of an equation gadget $H$ (see \Cref{def:coreTSPdef,def:eqGadget}). Later on in Appendix~\ref{sec:reductionTSP}, we will show how to invoke \Cref{thm:tsp2} with our equation gadget which achieves parameters $\sound(H)=\complete(H)=10$ to obtain \Cref{thm:tsp}.

As we mentioned in \Cref{sec:tsp}, our reductions will be described in terms of $\mcst$, which is an equivalent problem to metric TSP. In fact, instead of working directly with $\mcst$, it will be more convenient to work with a variation where certain edges are ``forced'' to be in the spanning tours we consider. 

\begin{defn}[$\tspforced$]
    An instance of $\tspforced$ is defined by a weighted graph $G = (V, E_u\cup E_f, w)$, where the edge set is split into a set $E_u$ of ``unforced'' edges, and $E_f$ of ``forced'' edges.
\end{defn}

Recall that a spanning tour is a closed walk on a graph that visits every vertex at least once. The objective of $\tspforced$ is: given an instance $G$, to find the minimum weight spanning tour in $G$ that uses each forced edge at least once (and without loss of generality, one that uses each edge at most twice). With this in mind, it will be convenient to make the following definition.

\begin{defn}[Valid spanning tour]
    Let $T$ be a spanning tour in a graph $G =(V, E_u\cup E_f, w)$. We say that $T$ is a \emph{valid} tour on $G$ if it (1) it uses each forced edge in $E_f$ at~\emph{least once}, and (2) it uses each edge in $G$ at most twice.
     \label{valid:sptn}
\end{defn}

A simple reduction shows that approximating $\tspforced$ is equivalent to approximating $\mcst$.

\begin{lem}[e.g. \cite{karpinski2015new}]\label{lem:tsp-to-forced}
    Let $\eps > 0$. There is a polynomial time reduction mapping an instance $G=(V, E_u\cup E_f, w)$ of $\tspforced$ to an instance $G' = (V, E', w')$ of $\mcst$, so the following inequalities hold, where $T$ is the minimum weight valid spanning tour in $G$ and $T'$ is the minimum weight spanning tour in $G'$.
    \[w(T) - \eps\cdot \max_{e\in E_f}w(e)\leq w'(T')\leq w(T).\]
\end{lem}

We will henceforth focus on proving inapproximability for $\tspforced$ by reduction from $\threelin{2}$. As in~\Cref{sec:tsp}, our reduction is defined by a gadget that determines how we map $\predeqv{1}{2}{3}$ clauses to edges in a $\tspforced$ instance, which we refer to as an~\emph{equation gadget}.

\begin{defn}[Equation gadget]
    An equation gadget (e.g.,~\Cref{subfig:tsp_AE_special}) is a weighted multigraph $H = (V, E_u\cup E_f, w)$ on vertex set $V = [3 + 1 + \naux]$ with the following properties:
    \begin{enumerate}
        \item \textbf{Contact, central, and auxiliary vertices}: The $[3+1+\naux]$ of vertices are partitioned into \emph{three} contact vertices $\{1,2,3\}$, one central vertex $4$, and $\naux$ auxiliary vertices $\{5,\ldots, 3+1+\naux\}$.
        \item \textbf{Unforced, forced, and special edges}: $E_u$ and $E_f$ are referred to as the \emph{unforced} and \emph{forced} edges respectively, therefore $H$ forms an instance of $\tspforced$. There is a fixed set $E_s\subset E_u$ of edges referred to as \emph{special} edges, consisting of edges $(\ell, 4)$ with weight $w((\ell, 4)) = 1/2$ for each contact vertex $\ell\in \{1,2,3\}$.
    \end{enumerate}
    \label{def:eqGadget}
\end{defn}
Given an equation gadget $H$ and a valid spanning tour $Q$ in $H$, we define
\begin{itemize}
\item $k_i^H(Q)$ to be the number of special edges that appear exactly $\inst$ times in $Q$. Note for example that $\sum_i k_i^H(Q)=3$ for all valid $Q$.
\end{itemize}

$H$ will be used to map the ``equals 1'' clause $\predeqv{1}{2}{3}(z_1, z_2, z_3) = \mathbf{1}\{z_1 + z_2 + z_3 \equiv 1\pmod 2\}$ of $\threelin{2}$ (recall that the variables $z_\ell$ are over $\{0,1\}$).
As intuition, we will aim to encode an assignment $\bfz\in \{0,1\}^3$ to the variables of a $\predeqv{1}{2}{3}$ clause as a spanning tour $T$ such that the special edge $(i, 4)$ is used exactly $2\cdot z_i$ times. We will require that satisfying assignments to $\predeqv{1}{2}{3}$ admit such spanning tours (which must have $k_1(T)=0$ and $k_2(T)$ odd) with low weight, and that spanning tours not of this form must be more expensive. We formally encapsulate this intuition in the following definition.

\begin{defn}[Soundness and Completeness] Here we provide the completeness and soundness definitions for any equation gadget $H = (V, E_u\cup E_f, w)$.

\begin{itemize}
\item \textbf{Completeness}: Let $\bfz\in \{0,1\}^3$, and let $C = (\predeqv{1}{2}{3})^{-1}(1)\subset \{0,1\}^3$ be the set of satisfying assignments of $\predeqv{1}{2}{3}$. We use $\mathbb{Q}_\bfz$ to refer to the set of valid spanning tours $Q$ in $H$ such that the special edge $(\ell, 4)$ is unused for all $\ell$ such that $\bfz_\ell = 0$, and is used twice otherwise. $\complete(H)$ is defined as
    \[\complete(H) = \max_{z\in C}\min_{Q\in \mathbb{Q}_\bfz} w(Q)\]
    
    \item \textbf{Soundness}: $\sound(H)$ is defined as
    \[\sound(H) = \min_{Q} w(Q) -  k_1^H(Q)/2 - \mathbf{1}\{k_1^H(Q)=0\}\cdot \mathbf{1}\{k_2^H(Q)\text{ is even}\},\]
    where the minimum is taken over all valid spanning tours $Q$.
    \end{itemize}
    \label{def:coreTSPdef}
\end{defn}

Similarly to \cite{chlebik2022weighted}, we will perform a reduction from a particular ``hybrid'' bounded-occurrence CSP over $\{0,1\}$ (described in~\Cref{thm:hybrid_CSP}). Informally, the hybrid CSP is a particularly structured instance that contains a small number of $\predeqv{1}{2}{3}$ clauses along with some simpler two-variable equality and inequality clauses, where $\predeqv{1}{2}{3}(z_1, z_2, z_3) = \mathbf{1}\{z_1 + z_2 + z_3 \equiv 1\pmod 2\}$. Consider the 2-ary predicates $\predeq(x_1, x_2) = \mathbf{1}\{x_1 = x_2\}$, $\predneq(x_1, x_2) = \mathbf{1}\{x_1\neq x_2\}$.  In~\Cref{def:weighted3} we provide the CSP instance we reduce from, and in~\Cref{thm:hybrid_CSP} we state its soundness and completeness guarantees.

\begin{defn}[$\wthreelintwo$~\cite{chlebik2022weighted}]\label{defn:wthreelin}
    Let $k$ be a fixed integer. An instance of $\wthreelintwo$ takes the form $\inst = \insteq + 2\cdot\instneq + 2\cdot\insti{1}{2}{3}$ of $\csp(\predeqv{1}{2}{3}, \predeq, \predneq)$ where:
    \begin{itemize}
        \item \textbf{Variables}: The variables in $\inst$ are labeled as $x_{i,j,b}$ for $i\in [n]$, $j\in [11k]$, and $b\in \{0,1\}$. That is, there are $22kn$ variables. Each variable will participate in exactly two distinct $\predeq$ clauses in $\insteq$. Let $S = \{11\ell : \ell\in [k]\}$ be the set of multiples of $11$ in $[11k]$. Following the notation in \cite{karpinski2015new}, we will classify a variable $x_{i,j,b}$ as a ``contact'' if $j\in S$, and a ``checker'' otherwise. A contact will additionally appear in exactly one $\predneq$ clause in $\instneq$, and a checker will appear in exactly one $\predeqv{1}{2}{3}$ clause in $\insti{1}{2}{3}$.
        \item \textbf{Equality clauses}: $\insteq$ consists of the following: each variable $x_{i,j,b}$ participates in two clauses $\predeq(x_{i,j,b}, x_{i,j+1,b})$ and $\predeq(x_{i,j,b}, x_{i, j-1, b})$, where addition/subtraction to $j$ is done modulo $11k$.
        \item \textbf{Inequality clauses}: $\instneq$ is defined by the following: for each $i\in [n]$, $j\in [11k]\setminus S$, we add the clause $\predneq(x_{i,j,0}, x_{i,\pi(j),1})$, where $\pi$ is a fixed bijection on $[11k]\setminus S$.
        \item \textbf{\threelin{2} clauses:} $\insti{1}{2}{3}$ consists of a collection of $\predeqv{1}{2}{3}$ clauses involving only contact variables, and every contact variable participates in exactly one such clause.
    \end{itemize}
    \label{def:weighted3}
\end{defn}

Note that in the definition above, $\inst$ contains one copy of each clause from $\insteq$, and two copies of each clause from $\instneq$ and $\insti{1}{2}{3}$.

\begin{thm}[Soundness and completeness for weighted-$\threelin{2}$~\cite{chlebik2022weighted}]
    Let $\varepsilon\in (0, 1/4)$, and let $k$ be a large enough integer depending on $\varepsilon$, and let $\inst$ be an instance of the weighted-$\threelin{2}$ instance from~\Cref{def:weighted3}. Then, it is NP-hard to decide between the two following cases:
    \begin{enumerate}
        \item \textbf{Completeness:} There is an assignment $\bfx = (x_{i,j,b})_{i,j,b}$ satisfying all clauses in $\insteq$ and $\instneq$, and all but $\varepsilon nk$ clauses in $\insti{1}{2}{3}$, or
        \item \textbf{Soundness:} For all assignments $\bfx = (x_{i,j,b})_{i,j,b}$, at least $(2/3-\varepsilon)nk$ clauses of $\inst = \insteq + 2\instneq + 2\insti{1}{2}{3}$ are unsatisfied.
    \end{enumerate}
    \label{thm:hybrid_CSP}
\end{thm}

We prove the following result, which can be interpreted as a modularized version of the proofs in \cite{karpinski2015new,chlebik2022weighted}.

\begin{thm}
    Let $H$ be an equation gadget (\Cref{def:eqGadget}) such that $\complete(H)<\infty$. For any $\varepsilon > 0$, it is NP-hard to approximate $\tspforced$ within $\frac{91 + 2\cdot\sound(H)}{90 + 2\cdot\complete(H)}-\varepsilon$.
    \label{thm:tsp2}
\end{thm}
\begin{proof}
    We will provide a reduction from the NP-hard problem described in \Cref{thm:hybrid_CSP} to $\tspforced$. Suppose we are given a $\wthreelintwo$ instance $\inst=\insteq+2\cdot\instneq + 2\cdot\insti{1}{2}{3}$ on variables $x_{i,j,b}$ for $i\in [n], j\in [11k], b\in \{0,1\}$. Our instance $G$ of $\tspforced$ will contain a vertex for each variable, which we will also denote by $x_{i,j,b}$. Additionally we will add a single ``central'' vertex $s$.  We will convert every clause of $\inst$ into a corresponding structure in $G$, which we define below. We also provide a visual representation of this correspondence in~\Cref{subfig:tsp_AE_special} and~\Cref{fig:cycle-gadget}.

    \begin{itemize}
        \item \textbf{Equality clauses}: Each clause in $\insteq$ will be replaced by a single unforced edge in $G$ of weight one. (In~\Cref{fig:cycle-gadget} these edges are denoted in black.)
        \item \textbf{Inequality clauses}: Each clause in $\instneq$ will be replaced by two parallel forced edges in $G$ of weight one. (In~\Cref{fig:cycle-gadget} these edges are denoted in red.)
        \item \textbf{\threelin{2} clauses}: A clause $\predeqv{1}{2}{3}(x_{i_1,j_1,b_1},x_{i_2,j_2,b_2},x_{i_3,j_3,b_3})$ in $\insti{1}{2}{3}$ will be replaced by a copy of the gadget $H$ (shown in~\Cref{subfig:tsp_AE_special}). More precisely, for each such clause we will add $\naux$ \emph{new} vertices, say $a_1,\ldots, a_{\naux}$, and add a single induced copy of the forced and unforced edges from $H$ on the contact vertices $x_{i_1,j_1,b_1},x_{i_2,j_2,b_2},x_{i_3,j_3,b_3}$, auxiliary vertices $a_1,\ldots, a_{\naux}$, and central vertex $s$. Note that the central vertex $s$ is shared among all $\predeqv{1}{2}{3}$ clauses, but the auxiliary vertices are individually created for each $\predeqv{1}{2}{3}$ clause.
    \end{itemize}

    We will define $\edgeeq$, $\edgeneq$, and $\edgei{1}{2}{3}$ as the edges added in the three steps above respectively. This completes the definition of $G$. It remains to argue that the $\tspforced$ value of $G$ has the desired gap in the completeness and soundness cases.

\begin{figure}[htb]
  \centering
  \begin{subfigure}{0.45\textwidth}
    \centering
    \includegraphics[scale=0.3]{images/pc1_special_new.png}
    \caption{The equation gadget discovered by AlphaEvolve. This figure is a reproduction of~\Cref{fig:tsp_AE_pc1}.}\label{subfig:tsp_AE_special}
  \end{subfigure}
  \hfill
  \begin{subfigure}{0.45\textwidth}
    \centering
    \includegraphics[scale=0.3]{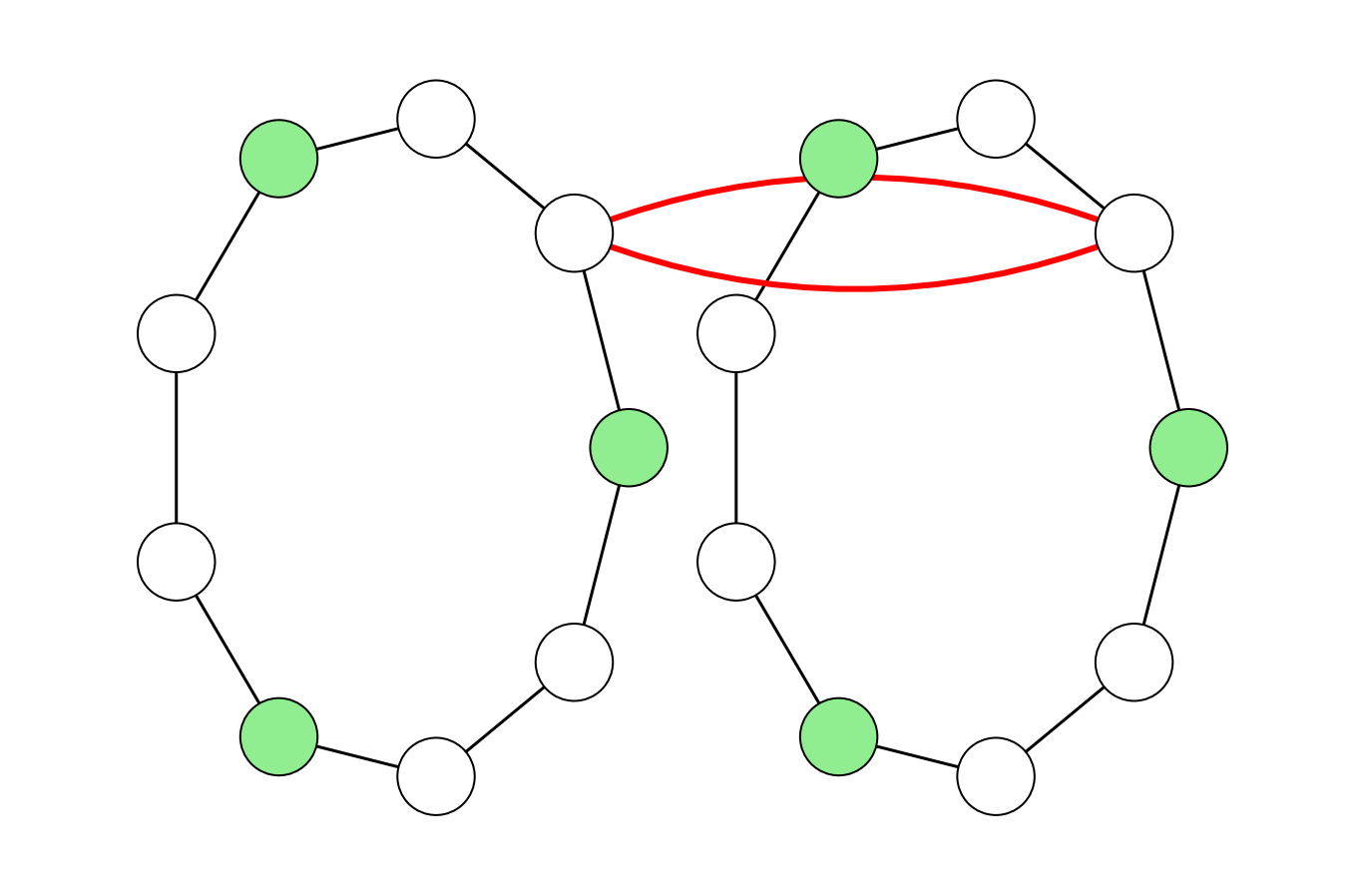}
    \caption{Graph encoding $\insteq$ and $\instneq$ clauses. The green vertices correspond to the contact vertices, and the white ones correspond to the checker vertices. The black edges encode the $\insteq$ clauses, and the {\color{red} red} edges encode the $\instneq$ clauses. All edges have weight one.}
    \label{fig:cycle-gadget}
  \end{subfigure}
  \caption{Equation gadget discovered by AlphaEvolve, and the graph encoding of the various clauses.}
  \label{fig_combined_AE_special}
\end{figure}


    \begin{lem}[Completeness]\label{lem:comp-tsp}
    Suppose there is an assignment $\bfx = (x_{i,j,b})_{i,j,b}$ to the $\wthreelintwo$ instance $\inst$ that satisfies all clauses in $\insteq$ and $\instneq$ clauses, and violates at most $\Delta$ clauses in $\insti{1}{2}{3}$. Then, there is a valid spanning tour in $G$ of total weight at most $nk\cdot (30 + 2\cdot\complete(H)/3) +  2(\Delta + 2n)\cdot w(H)$.
    \label{lem:completeness}
    \end{lem}

    \begin{lem}[Soundness]\label{lem:sound-tsp}
        Suppose there is a valid spanning tour $T$ of the $\tspforced$ instance $G$ with cost at most $nk\cdot (30 + 2\cdot\sound(H)/3) + \Delta$. Then there is an assignment $\bfx=(x_{i,j,b})_{i,j,b}$ that leaves at most $2\Delta$ clauses in $\inst$ unsatisfied.
    \end{lem}

    Before proving \Cref{lem:comp-tsp,lem:sound-tsp}, we will use them to complete the proof of \Cref{thm:tsp2}. \Cref{thm:hybrid_CSP} combined with \Cref{lem:comp-tsp,lem:sound-tsp} implies that it is NP-hard to distinguish between instances $G$ that admit a spanning tour of weight at most $nk\cdot (30 + 2\cdot\complete(H) / 3) + 2(\varepsilon nk + 2n)\cdot w(H)$ from instances where all spanning tours have weight at least $nk\cdot (30 + 2\cdot\sound(H) / 3) + nk\cdot(1/3-\varepsilon/2)$. That is, it is NP-hard to approximate $\tspforced$ within 
    \[\frac{nk\cdot (30 + 2\cdot\sound(H) / 3) + nk\cdot(1/3-\varepsilon/2)}{nk\cdot (30 + 2\cdot\complete(H) / 3) + 2(\varepsilon nk + 2n)\cdot w(H)}.\]
    Setting $\varepsilon$ to be sufficiently small and $k$ to be sufficiently large, this is at least
    \[\frac{91 + 2\cdot\sound(H)}{90 + 2\cdot\complete(H)}-\varepsilon'\]
    for arbitrarily small $\varepsilon'$. It remains to prove \Cref{lem:comp-tsp,lem:sound-tsp}, which we do in the following section.
\end{proof}

\subsection{Proofs of~\Cref{lem:comp-tsp,lem:sound-tsp}}
    \begin{proof}[Proof of \Cref{lem:comp-tsp}]
        We will construct collections of edges $\edgeteq,\edgetneq$, and $\edgeti{1}{2}{3}$ from $\edgeeq, \edgeneq$, $\edgei{1}{2}{3}$, and set $T = \edgeteq\cup \edgetneq\cup \edgeti{1}{2}{3}$. First let us specify how we construct $\edgeteq$ and $\edgetneq$.

        \begin{itemize}
            \item We add a single copy of each (forced) edge in $\edgetneq$. Note that there are two parallel forced edges corresponding to every inequality constraint in $\instneq$; we add a single copy of both.
            \item For a variable set to $1$, $\edgeteq$ contains exactly one copy of both unforced edges in $\edgeeq$ adjacent to the corresponding vertex in $G$.
        \end{itemize}

        Note that because every $\insteq$ and $\instneq$ clause is satisfied, for each $i\in [n]$ there is a single $b^*\in \{0,1\}$ such that $x_{i,j,b} = \mathbf{1}\{b = b^*\}$. So $\edgeteq$ contains exactly half the edges in $\edgeeq$. We can hence calculate $w(\edgetneq) = 20nk$, and $w(\edgeteq) = 11nk$. Next we describe how to construct $\edgeti{1}{2}{3}$.
        
        For each distinct clause of the form $\predeqv{1}{2}{3}(x_{i_1,j_1,b_1},x_{i_2,j_2,b_2}, x_{i_3,j_3,b_3})$, we consider the optimal spanning tour $Q^* = \arg\min_{Q\in \mathbb{Q}_\bfz}w(T)$, where $\bfz = (x_{i_1,j_1,b_1},x_{i_2,j_2,b_2}, x_{i_3,j_3,b_3})$. We will add all \emph{non-special} edges of $Q^*$ on the corresponding contact vertices $\{x_{i_1,j_1,b_1},x_{i_2,j_2,b_2}, x_{i_3,j_3,b_3}\}$, central vertex $s$, and auxiliary vertices $a_1,\ldots, a_{\naux}$ specific to the clause.
        
        As an exception to the above, in the case that either of $j_1$, $j_2$, or $j_3$ is equal to $11k$, we will instead add two copies of \emph{all} non-special edges, rather than just the edges in $Q^*$. The purpose of this step is to ensure connectivity of the tour; it has minimal impact on the cost of the tour, as $k$ should be thought of as very large, and this step affects $o_k(1)$ fraction of the edges. We can bound the weight of edges added by
        \[\complete(H) - (x_{i_1,j_1,b_1}+x_{i_2,j_2,b_2}+x_{i_3,j_3,b_3}) + 2 w(H)\cdot \left(\mathbf{1}\{11k\in \{j_1,j_2,j_3\}\} + \mathbf{1}\{\predeqv{1}{2}{3}(x_{i_1,j_1,b_1},x_{i_2,j_2,b_2}, x_{i_3,j_3,b_3})=0\} \right),\]
        where the second term accounts for the fact that we removed two special edges (each having weight $1/2$) for each variable that was set to $1$.

        Summing over all clauses in $\insti{1}{2}{3}$, we get
        \[w(\edgeti{1}{2}{3})\leq \complete(H)\cdot 2nk/3 - nk + 2w(H)\cdot (2n + \Delta),\]
        where we used that the number of $\insti{1}{2}{3}$ clauses is $2nk/3$, the fact that at most $2n$ clauses satisfy $11k\in \{j_1,j_2,j_3\}$, and at most $\Delta$ clauses are unsatisfied.
        
        To complete the proof of~\Cref{lem:comp-tsp}, we argue that $T$ is indeed a valid spanning tour of $G$, which we do in the following two claims.
        \begin{claim}
            The number of edges of $T$ adjacent to each vertex in $G$ is even. 
        \end{claim}
        \begin{proof}
            The number of edges of $\edgeteq$ and $\edgetneq$ adjacent to each vertex in $G$ is even by construction.

            Every non-contact vertex is also adjacent to an even number of edges in $\edgeti{1}{2}{3}$. Finally, since in every spanning tour in $\mathbb{Q}_\bfz$ for $\bfz\in \{0,1\}^3$, each special edge (which we abstained from adding to $\edgeti{1}{2}{3}$) is used an even number of times, it also holds that every contact vertex is adjacent to an even number of edges in $\edgeti{1}{2}{3}$.
        \end{proof}
        
        \begin{claim}
            $T$ has a single connected component consisting of all vertices in $G$.
        \end{claim}
        \begin{proof}
            We will argue that each vertex is connected to the central vertex $s$.
            \begin{itemize}
                \item The contact vertices $x_{i,j, b}$ are connected to $s$ using only edges in $\edgeti{1}{2}{3}$ if either (1) $x_{i,j,b}=0$, or (2) $j=11k$.
                \item All checker vertices, and contact vertices set to $1$ are connected to either $x_{i,11k, 0}$ or $x_{i, 11k, 1}$ by edges in $\edgeteq$ and $\edgetneq$. Therefore they are also connected to $s$ in $T$ by the previous step.
                \item All auxiliary vertices in $H$ are connected to all contact vertices as well as the central vertex $4$ in any tour in $Q\in \mathcal{T}_\bfz$ for an assignment $\bfz\in \{0,1\}^3$. Consequently, they are also connected to either $4$ or a contact vertex after removing any special edges from the $Q$. As a result, all auxiliary vertices in $G$ are connected to either $s$, or to a contact vertex, which have all been previously shown to be connected to $s$ using edges in $T$.
            \end{itemize}
        \end{proof}

        By construction, each forced edge is used at least once, and no edge is used more than twice. Therefore $T$ is a valid spanning tour of $G$ of weight at most $nk\cdot (30 + 2\cdot\complete(H)/3) + 2(\Delta + 2n)\cdot w(H)$.
    \end{proof}
    
    \begin{proof}[Proof of \Cref{lem:sound-tsp}]
        Let us write $T = \edgeteq\cup \edgetneq\cup \edgeti{1}{2}{3}$, where $\edgeteq, \edgetneq$, and $\edgeti{1}{2}{3}$ are the edges in $T$ belonging to $\edgeeq$, $\edgeneq$, and $\edgei{1}{2}{3}$ respectively. We define $\edgeueq$ to be $\edgeteq$ with all edges appearing twice removed.
    
        We will assign $x_{i,j,b}$ based on the ``local'' view of the variable $x_{i,j,b}$ as a vertex in $\edgeueq$. In particular, let $l_{i,j,b}$ and $r_{i,j,b}$ denote the number of copies of the edges $(x_{i,j-1, b}, x_{i,j,b})$ and $(x_{i,j,b}, x_{i,j+1,b})$ in $\edgeteq$ respectively. We say that the variable $x_{i,j,b}$ is ``honest'' if $l_{i,j,b} = r_{i,j,b}$, and ``dishonest'' otherwise. We will always assign an honest variable $x_{i,j,b} = l_{i,j,b}$. Note that this is equivalent to the assignment $x_{i,j,b} = r_{i,j,b}$.
        
        Next, we describe how to assign the dishonest variables. We will do this by specifying how to assign the dishonest variables appearing in each inequality and $\threelin{2}$ clause.
        \begin{itemize}
            \item For each inequality clause $\predneq(x_{i,j,0}, x_{i,M(j), 1})$ in $\instneq$, note that the variables involved are either both honest, or both dishonest. In the case that both are dishonest, we assign each variable a random value in $\{0,1\}$, conditioned on $x_{i,j,0}\neq x_{i,M(j),1}$.
            \item For each $\threelin{2}$ clause $\predeqv{1}{2}{3}(x_{i_1,j_1,b_1}, x_{i_2,j_2,b_2},x_{i_3,j_3,b_3})$ in $\insti{1}{2}{3}$ containing at least one dishonest variable, we arbitrarily set the dishonest variables involved, while ensuring that the clause is satisfied. Note that this is always possible as long as there is at least one such dishonest variable.
        \end{itemize}
        In order to analyze the number of unsatisfied clauses, we will need to define a few quantities.
        \begin{itemize}
            \item $d_{ch}$ and $d_{co}$ denote the number of dishonest checker and contact vertices respectively.
            \item $p_{ch}$ denotes the number of inequality constraints in $\instneq$ such that both involved checkers are honest and assigned the same value.
            \item $r$ is the number of unsatisfied clauses in $\insti{1}{2}{3}$ where all involved variables are honest.
        \end{itemize}

        \begin{claim}\label{claim:tsp-soundness-1}
            The expected number of unsatisfied clauses in $\inst$ is at most $d_{ch} + d_{co} + 2p_{ch} + 2u_{ch} + 2r$.
        \end{claim}
        \begin{proof}
            Let us list the collection of unsatisfied clauses. Using~\Cref{def:weighted3}, we have the following:
            \begin{itemize}
                \item \textbf{Equality clauses:} By definition, the equality clause between two honest vertices is always satisfied. Every equality clause adjacent to a dishonest checker vertex is satisfied with probability exactly $1/2$ due to the random assignment of dishonest checkers.
                
                Out of the remaining edges, we claim that there is at most one unsatisfied edge adjacent to each dishonest contact vertex $x_{i,j,b}$. This is because if both $x_{i,j-1,b}$ and $x_{i, j+1, b}$ are honest, they must be assigned different values, as $r_{i,j-1,b} = l_{i,j,b}$ has parity opposite of $l_{i,j+1,b} = r_{i,j,b}$, since $x_{i,j,b}$ is dishonest.
                
                So, the expected number of violated equality clauses is at most $d_{ch} + d_{co}$.
                \item \textbf{Inequality clauses:} An inequality clause in $\instneq$ is unsatisfied when both checkers it involves are honest and assigned the same value. So, the number of violated inequality clauses is equal to $p_{ch}$.
                \item \textbf{\threelin{2} clauses:} By definition, all clauses in $\insti{1}{2}{3}$ involving dishonest variables are satisfied. So the number of unsatisfied clauses in $\insti{1}{2}{3}$ is equal to $r$.
            \end{itemize}

            Since the instance $\inst$ is defined as $\insteq + 2\instneq + 2\insti{1}{2}{3}$, the above bounds imply that the number of unsatisfied clauses in $\inst$ is at most $d_{ch} + d_{co} + 2p_{ch} + 2r$.
        \end{proof}        

        \begin{claim}\label{claim:tsp-soundness-2}
            $w(T) \geq nk\cdot (30 + 2\cdot\sound(H)/3) + (d_{ch} + d_{co})/2 + p_{ch} + r$. 
        \end{claim}
        \begin{proof}
            We will separately bound the weight of edges in $\edgeteq$, $\edgetneq$, and $\edgeti{1}{2}{3}$. Let us denote by $t_{co}$ the number of honest contact variables assigned $1$. For $b\in \{0,1\}$ let $p_{ch}^b$ denote the number of inequality clauses such that both involved checkers are assigned $b$. Note that $p_{ch}^0 + p_{ch}^1 = p_{ch}$. Finally, let $u_{co}$ be the number of connected components in $\edgeueq\cup \edgetneq\cup \edgeti{1}{2}{3}$ containing only edges from $\edgei{1}{2}{3}$.
            \begin{itemize}
                \item \textbf{Bounding $\edgeueq$:} By double-counting, we can write
                \[
                    w(\edgeueq) = \frac{1}{2}\sum_{i,j,b}\left(l_{i,j,b}+r_{i,j,b}\right)
                    = \frac{1}{2}\left(\sum_{i,j:j\mid 11,b}\left(l_{i,j,b}+r_{i,j,b}\right) + \sum_{i,j:j\nmid 11}v_{i,j}\right),
                \]
                where $v_{i,j}=\left(l_{i,j,0}+r_{i,j,0}+l_{i,M(j),1}+r_{i,M(j),1}\right)$. We will begin by bounding the first sum, that is, the sum over all contact variables. Observe that a dishonest contact variable must have $l_{i,j,b}+r_{i,j,b}\geq 1$, since $l_{i,j,b} \not\equiv r_{i,j,b}\pmod 2$. Similarly, a honest variable assigned $1$ must have $l_{i,j,b}+r_{i,j,b}\geq 2$. So the first sum is at least $d_{co}+2t_{co}$.

                Next, we will bound the second sum. We have that the sum of degrees of $x_{i,j,0}$ and $x_{i,M(j),1}$ in $\edgeueq$ is even; equivalently, $v_{i,j}$ must be even. So, $v_{i,j}\geq 2$ unless both involved checkers are assigned $0$. Furthermore, we have $v_{i,j}\geq 4$ if both involved checkers are honest and assigned $1$.
                
                Therefore 
                \[\sum_{i\in [n],j\in [11k]:j\nmid 11}v_{i,j}\geq 2\cdot 10nk + 2\cdot p_{ch}^1 - 2\cdot p_{ch}^0.\]
                 
                Together, these bounds imply  $w(\edgeueq)\geq 10nk + d_{co}/2 + t_{co} + p_{ch}^1 - p_{ch}^0$.
                \item \textbf{Bounding $\edgeteq\setminus \edgeueq$:} Since $T=\edgeteq\cup \edgetneq\cup \edgeti{1}{2}{3}$ is a spanning tour of $G$, we have that the total weight of edges in $\edgeteq\setminus \edgeueq$ is at least twice the number of connected components of $\edgeueq\cup \edgetneq\cup \edgeti{1}{2}{3}$. So, $w(\edgeteq\setminus \edgeueq)\geq 2p_{ch}^0 + 2u_{co}$.
                \item \textbf{Inequality clauses:} Define $\kappa_{i,j}$ to be the total number of copies of both forced edges between $x_{i,j,0}$ and $x_{i,M(j), 1}$. We can write
                \[w(\edgetneq) = \sum_{i\in [n],j\in [11k]:j\nmid 11}\kappa_{i,j}.\]
                Since $T$ was a valid tour, we have $\kappa_{i,j}\geq 2$ for all such $i,j$. Furthermore, if either of the involved vertices is dishonest, the fact that the degree of that vertex in $T$ is even implies that $\kappa_{i,j}$ is odd. As a consequence, we have $\kappa_{i,j}\geq 3$ for at least $d_{ch}/2$ such $(i,j)$. This implies $w(\edgetneq) \geq 2\cdot 10nk + d_{ch}/2$.
                
                \item \textbf{\threelin{2} clauses:} Let $\clause = \predeqv{1}{2}{3}(x_{i_1,j_1,b_1}, x_{i_2,j_2,b_2},x_{i_3,j_3,b_3})$ be a clause in $\insti{1}{2}{3}$. Let $T_\clause$ be the set of edges in $\edgeti{1}{2}{3}$ involved in clause $\clause$. We will extract a valid spanning tour $Q$ of $H$ from $T_\clause$, and then use the soundness property of $H$ to argue that $w(T_\clause)$ is large.

                Formally, we will map the edges in $T_\clause$ to a multiset of edges $Q$ of the gadget. Note that $Q$ is not necessarily a spanning tour of $H$ yet. To resolve this, we add $l_{i_{\ell},j_{\ell},b_{\ell}}+r_{i_{\ell},j_{\ell},b_{\ell}}$ copies of the special edge $(\ell, 4)$ for each $\ell\in [3]$. For each remaining connected component that is not connected to the central vertex, we will add two copies of the special edge $(\ell, 4)$ for an arbitrary contact vertex $\ell$ in that connected component. Let us denote by $\localcomps$ the number of such components.

                Observe that the degree of all non-central vertices in $Q$ is identical to the degree of the corresponding vertex in $T$. As a consequence, all non-central vertices have even degree in $Q$. Since the sum of degrees of vertices in a multigraph is always even, we automatically get that all vertices have even degree in $Q$. Furthermore, by definition $Q$ is connected. Also note that $Q$ uses each edge at most twice, and uses every forced edge at least once.

                So, $Q$ is indeed a valid spanning tour in $H$, implying $w(Q)$ is at least
                \[\sound(H) + k_1^H(Q)/2 + \mathbf{1}\{k_1^H(Q)=0\}\cdot \mathbf{1}\{k_2^H(Q)\text{ is even}\}.\]

                Subtracting the contribution of the special edges we added to $Q$, we get
                \begin{align*}
                    w(T_\clause)&= w(Q) - k_1^H(Q)/2 - k_2^H(Q)\\
                    &\geq \sound(H) - \localcomps - (k_2^H(Q) - \localcomps)  +\mathbf{1}\{k_1^H(Q)=0\}\cdot \mathbf{1}\{k_2^H(Q)\text{ is even}\}.
                \end{align*}

                Note that if $\localcomps=0$, the indicator $\mathbf{1}\{k_1^H(Q)=0\}\cdot \mathbf{1}\{k_2^H(Q)\text{ is even}\}$ exactly indicates whether $\clause=1$, that is, whether the clause is satisfied. Therefore, the above is at least
                \[s(H) - 2\localcomps - (k_2^H(Q) - \localcomps) + \clause\]
                
                Summing over all $2nk/3$ clauses $c$ in $\insti{1}{2}{3}$, we get
                \[w(\edgeti{1}{2}{3})\geq (2nk/3)\cdot \sound(H) - 2u_{co} - t_{co} + r,\]
                where we used that the sum of $\localcomps$ equals $u_{co}$, along with the fact that the sum of $k_2^H(Q) - \localcomps$ is the total number of contact vertices $x_{i,j,b}$ satisfying $l_{i,j,b}=r_{i,j,b}=1$, which equals $t_{co}$.
                
            \end{itemize}

            Together, these imply the lower bound $w(T) = w(\edgeueq) + w(\edgeteq) + w(\edgetneq) + w(\edgeti{1}{2}{3})\geq nk\cdot (30 + 2\cdot\sound(H)/3) + (d_{ch} + d_{co})/2 + p_{ch} + r$.
        \end{proof}
        Since $T$ has weight at most $nk\cdot (30 + 2\cdot\sound(H)/3)+\Delta$, \Cref{claim:tsp-soundness-2} implies that $(d_{ch}+d_{co})/2+p_{ch}+r\leq \Delta$. \Cref{claim:tsp-soundness-1} directly shows that at most $d_{ch}+d_{co}+2p_{ch}+2r\leq 2\Delta$ clauses are unsatisfied in expectation over the randomness in the assignment $\bfx$; this proves that there exists an assignment leaving at most $2\Delta$ clauses unsatisfied.
    \end{proof}
    
\subsection{Description of our gadget}
\label{sec:reductionTSP}

In this section, we provide the equation gadget (\Cref{def:eqGadget} and~\Cref{subfig:tsp_AE_special}) found by AlphaEvolve, and instantiate \Cref{thm:tsp2} with this gadget to obtain \Cref{thm:tsp}. The gadget $H=([3+1+\naux], E_u\cup E_f, w)$ will have $\naux=4$ auxiliary vertices. As in Appendix~\ref{sec:gadget-description}, we will specify the edges in the gadget by a weighted edge list, where an edge is represented as a tuple $(a,b,w(a,b))$, representing an edge between $a$ and $b$ of weight $w(a,b)$. The unforced edges $E_u$ contain the special edges (see \Cref{def:eqGadget}), along with the following non-special unforced edges.
{
\begin{lstlisting}
[(5, 6, 1.0), (6, 7, 1.0), (7, 8, 1.0), (8, 5, 1.0)]
\end{lstlisting}
}
The forced edges $E_f$ are defined as follows.
{
\begin{lstlisting}
[(1, 5, 1.0), (1, 5, 1.0), (2, 6, 1.0), (2, 6, 1.0), (3, 7, 1.0), (3, 8, 1.0)]
\end{lstlisting}
}
Below we calculate the various parameters associated with $H$.

\begin{lem}\label{lem:tsp-gadget-calc}
    Let $H$ be the equation gadget defined above. We have $\complete(H) = \sound(H)=10$, and $\max_{e\in E_u\cup E_f} w(e) = 2$. 
\end{lem}
We note that this is in alignment with the discussion in \Cref{sec:tsp}; every satisfying assignment to $\predeqv{1}{2}{3}$ admits a tour of cost at most $\complete(H) = 10$, whereas for every tour $Q$ corresponding to an unsatisfying assignment, we have $k_1^H(Q)=0$ and $k_2^H(Q)$ is even; hence $Q$ has weight at least $\sound(H) + \mathbf{1}\{k_1^H(Q)=0\}\cdot \mathbf{1}\{k_2^H(Q)\text{ is even}\} = 11$. We can now prove \Cref{thm:tsp}.

\begin{proof}[Proof of \Cref{thm:tsp}]
    We apply \Cref{thm:tsp2} with the above gadget, and use \Cref{lem:tsp-gadget-calc} to conclude that it is NP-hard to approximate $\tspforced$ within
    \[\frac{91+2\cdot\sound(H)}{90+2\cdot\complete(H)}-\varepsilon' = \frac{111}{110}-\varepsilon'\]
    for any $\varepsilon' > 0$. Using \Cref{lem:tsp-to-forced} along with the fact that $\mcst$ is equivalent to metric TSP, we conclude that it is also NP-hard to approximate metric TSP within $111/110-\varepsilon$ for any $\varepsilon > 0$.
\end{proof}

\subsection{Equation gadget for predicate $x+y+z\equiv 0\pmod 2 $}
\label{app:pc0_tsp}

In this part, we provide a detailed comparison of our gadgets with those in~\cite{chlebik2022weighted}. As mentioned in \Cref{sec:tsp}, \cite{chlebik2022weighted} operated with a version of $\wthreelintwo$ containing of $\predeqv{0}{2}{3}$ clauses (defined as $\predeqv{0}{2}{3}(x,y,z) = \mathbf{0}\{x + y + z \equiv 0\pmod 2\}$) rather than $\predeqv{1}{2}{3}$. As a CSP, this is completely equivalent to our $\wthreelintwo$ instances containing $\predeqv{1}{2}{3}$ clauses applied to the pointwise negation of the variables. However, as the $\{0,1\}$-assignment to a variable is encoded in a way that is not symmetric under negation~(see \Cref{def:eqGadget,def:coreTSPdef}), one needs a formally different notion of soundness and completeness in their setting. Since this notion is not critical to the correctness of our main results, we informally describe them. For an equation gadget $H$, the modified completeness $c'(H)$ is defined by maximizing over all $z\in (\predeqv{0}{2}{3})^{-1}(1)$ rather than $C=(\predeqv{1}{2}{3})^{-1}(1)$ as in \Cref{def:coreTSPdef}. The modified soundness $s'(H)$ is defined by replacing the $\mathbf{1}\{k_2^H(Q)\text{ is even}\}$ term in the definition of $s(H)$ by $\mathbf{1}\{k_2^H(Q)\text{ is odd}\}$.

Similarly to \Cref{thm:tsp2}, one can show that it is NP-hard to approximate $\tspforced$ within $(91+2\cdot s'(H))/(90+2\cdot c'(H))-\varepsilon$ for any equation gadget $H$. \cite{chlebik2022weighted} provide an equation gadget $H$~(see \Cref{subfig:tsp1_pc0}) achieving $s'(H)=c'(H)=13$, leading to their inapproximability of $117/116-\varepsilon$ for metric TSP.

As a final note, we used AlphaEvolve to find an equation gadget $H=([3+1+\naux], E_u\cup E_f, w)$~(see \Cref{fig:tsp_AE_pc0}) for $\predeqv{0}{2}{3}$ clauses on $\naux=8$ auxiliary vertices. Below we describe $H$ in the same format as Appendix~\ref{app:pc0_tsp}. The  unforced edges in $E_u$ are the following.
{
\begin{lstlisting}
[(5, 6, 1.0), (7, 8, 1.0), (10, 11, 1.0), (11, 12, 1.0), (12, 9, 1.0), (5, 10, 1.0), (6, 11, 1.0)]
\end{lstlisting}
}
The forced edges $E_f$ are defined as follows.
{
\begin{lstlisting}
[(1, 5, 1.0), (2, 6, 1.0), (3, 7, 0.0), (4, 8, 0.0), (5, 9, 0.0), (7, 11, 1.0), (8, 12, 0.0), (9, 1, 1.0), (10, 2, 1.0), (11, 3, 1.0)]
\end{lstlisting}
}
This gadget achieves $s'(H)=c'(H)=10$, giving an alternative approach to prove our main result~\Cref{thm:tsp}. However, this gadget is significantly more complicated than the $\predeqv{1}{2}{3}$-based equation gadget from \Cref{lem:tsp-gadget-calc} achieving $s(H)=c(H)=10$, which is why we chose to write our proofs in terms of $s(H)$ and $c(H)$.

\begin{figure}[htb]
  \centering
  \begin{subfigure}{0.45\textwidth}
    \centering
    \includegraphics[scale=0.3]{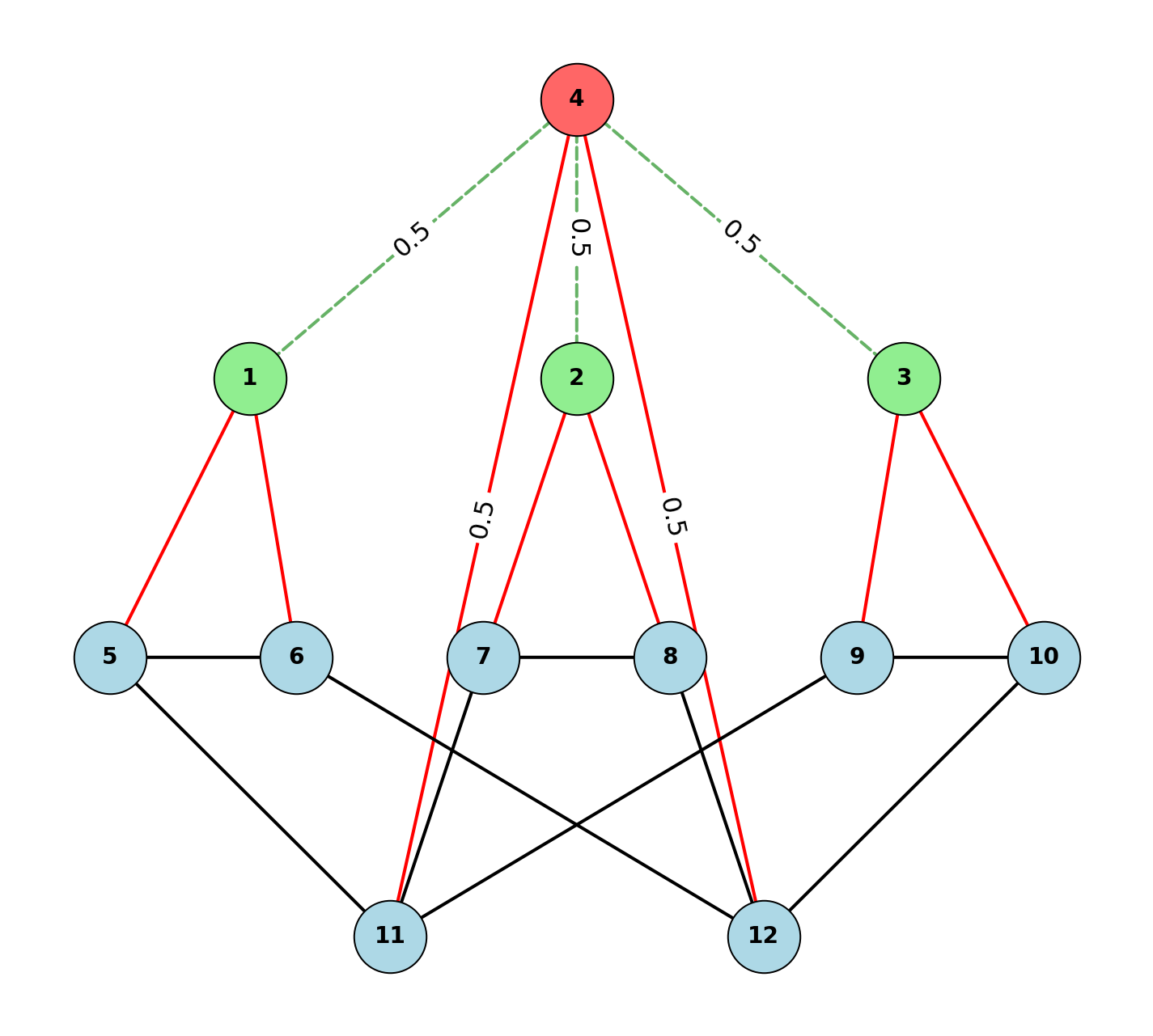}
    \caption{The equation gadget from~\cite{chlebik2022weighted}.}\label{subfig:tsp1_pc0}
  \end{subfigure}
  \hfill
  \begin{subfigure}{0.45\textwidth}
    \centering
    \includegraphics[scale=0.3]{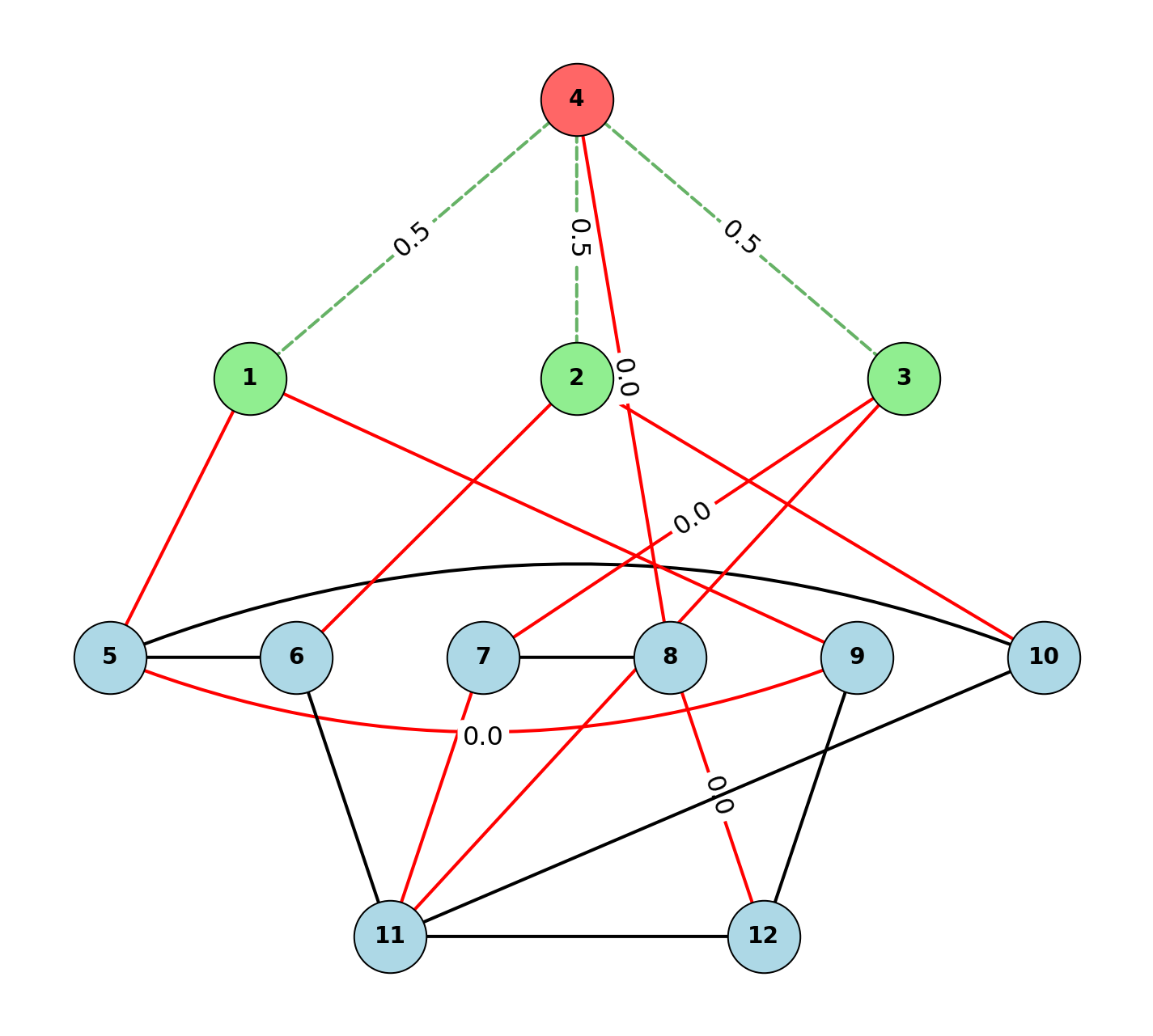}
    \caption{Equation gadget found by AlphaEvolve.}
    \label{fig:tsp_AE_pc0}
  \end{subfigure}
  \caption{Equation gadgets in the $\tspforced$ instance.  Vertices $\{1, 2, 3\}$ represent variables in the $\threelin{2}$ equation $\predeqv{0}{2}{3}$. The {\color{red} red} edges edges represent the forced edges. The dashed {\color{Green} green} edges represent the special edges. All edges without an explicitly labeled weight have weight $1$.}
  \label{fig:combined_tsp_gadgets}
\end{figure}

\end{document}